\documentclass{article}

\usepackage[preprint,nonatbib]{neurips_2021}

\usepackage[utf8]{inputenc} % allow utf-8 input
\usepackage[T1]{fontenc}    % use 8-bit T1 fonts
\usepackage{hyperref}       % hyperlinks
\usepackage{url}            % simple URL typesetting
\usepackage{booktabs} % for professional tables
\usepackage{amsmath}
\usepackage{amssymb}
\usepackage{amsthm}
\usepackage{bm}
\usepackage{bbm}
\usepackage{amsfonts}       % blackboard math symbols
\usepackage{nicefrac}       % compact symbols for 1/2, etc.
\usepackage{microtype}      % microtypography
\usepackage{bm}
\usepackage{multirow}

\usepackage{wrapfig}
\usepackage{subfig}
\usepackage{floatrow}
\usepackage{colortbl}
\usepackage{color}
\usepackage{soul}

\usepackage[textsize=scriptsize]{todonotes}
\usepackage[normalem]{ulem}

\definecolor{grey}{rgb}{0.8,0.8,0.8}
\definecolor{aqua}{rgb}{0, 1, 1}
\definecolor{steel}{rgb}{0.2734, 0.5078, 0.7031}
\definecolor{slate}{rgb}{0.1836, 0.3086, 0.3086}

\newcommand{\hlr}[2]{\setlength{\fboxsep}{0.3pt}\colorbox{red!#2}{\rule[-.05\baselineskip]{0pt}{.7\baselineskip}{#1}}}
\newcommand{\hlb}[2]{\setlength{\fboxsep}{0.3pt}\colorbox{aqua!#2}{\rule[-.05\baselineskip]{0pt}{.7\baselineskip}{#1}}}

\newcommand{\e}[1]{{\small $#1$}}
\DeclareMathOperator*{\argmin}{\arg\!\min}

\newcommand{\algname}{\textsc{A2r }}
\newcommand{\algnamens}{\textsc{A2r}}

\newtheorem{theorem}{Theorem}

\newfloatcommand{capbtabbox}{table}[][\FBwidth]

\allowdisplaybreaks

\title{Understanding Interlocking Dynamics of\\ Cooperative Rationalization}

\author{%
  Mo Yu$^{1}$\thanks{Authors contributed equally to this paper. Work was done when SC was at MIT-IBM Watson AI Lab.} \qquad Yang Zhang$^{1*}$ \qquad Shiyu Chang$^{1,2*}$ \qquad Tommi S. Jaakkola$^{3}$ \\
  $^1$MIT-IBM Watson AI Lab \qquad $^2$UC Santa Barbara \qquad $^3$CSAIL MIT\\
  \texttt{yum@us.ibm.com}\quad \texttt{yang.zhang2@ibm.com} \quad \texttt{chang87@ucsb.edu} \\ \texttt{tommi@csail.mit.edu}}

\begin{document}

\maketitle

\begin{abstract}

Selective rationalization explains the prediction of complex neural networks by finding a small subset of the input that is sufficient to predict the neural model output. The selection mechanism is commonly integrated into the model itself by specifying a two-component cascaded system consisting of a rationale generator, which makes a binary selection of the input features (which is the rationale), and a predictor, which predicts the output based only on the selected features. The components are trained jointly to optimize prediction performance. In this paper, we reveal a major problem with such cooperative rationalization paradigm --- \emph{model interlocking}. Interlocking arises when the predictor overfits to the features selected by the generator thus reinforcing the generator's selection even if the selected rationales are sub-optimal. The fundamental cause of the interlocking problem is that the rationalization objective to be minimized is concave with respect to the generator’s selection policy. We propose a new rationalization framework, called \algnamens, which introduces a third component into the architecture, a predictor driven by soft attention as opposed to selection. The generator now realizes both soft and hard attention over the features and these are fed into the two different predictors. While the generator still seeks to support the original predictor performance, it also minimizes a gap between the two predictors. As we will show theoretically, since the attention-based predictor exhibits a better convexity property, \algname can overcome the concavity barrier. Our experiments on two synthetic benchmarks and two real datasets demonstrate that \algname can significantly alleviate the interlock problem and find explanations that better align with human judgments.\footnote{We release our code at \url{https://github.com/Gorov/Understanding_Interlocking}.}

\end{abstract}

%-------------------------------------------------------------------------
\section{Introduction}
\label{sec:intro}

Selective rationalization \cite{bastings2019interpretable,carton2018extractive,chang2019game, chen2018learning,chen2018shapley,glockner2020you,lei2016rationalizing, li2016understanding,yu2019rethinking} explains the prediction of complex neural networks by finding a small subset of the input -- rationale -- that suffices on its own to yield the same outcome as to the original data.  To generate high-quality rationales, existing methods often train a cascaded system that consists of two components, \emph{i.e.}, a \emph{rationale generator} and a \emph{predictor}.  The generator selects a subset of the input explicitly (\emph{a.k.a.}, binarized selection), which is then fed to the predictor.  The predictor then predicts the output based only on the subset of features selected by the generator.  The rationale generator and the predictor are trained jointly to optimize the prediction performance.   Compared to many other interpretable methods \cite{bahdanau2014neural,kim2017structured,xu2015show,serrano2019attention} that rely on attention mechanism as a proxy of models' explanation, selective rationalization offers a unique advantage: certification of exclusion, \emph{i.e.}, any unselected input is guaranteed to have no contribution to prediction.  

However, binarized selective rationalization schemes are notoriously hard to train \cite{bastings2019interpretable,yu2019rethinking}.  To overcome training obstacles, previous works have considered using smoothed gradient estimations (\emph{e.g.} gradient straight-through \cite{bengio2013estimating} or Gumbel softmax \cite{jang2016categorical}), introducing additional components to control the complement of the selection \cite{carton2018extractive,yu2019rethinking}, adopting different updating dynamics between the generator and the predictor \cite{chang2019game}, using rectified continuous random variables to handle the constrained optimization in training~\cite{bastings2019interpretable}, \emph{etc.}
In practice, these solutions are still insufficient.  They either still require careful tuning or are at a cost of reduced predictive accuracy.  

In this paper, we reveal a major training problem of selective rationalization that has been largely overlooked --- \emph{model interlocking}.  Intuitively, this problem arises because the predictor only sees what the generator selects during training, and tends to overfit to the selection of the generator. As a result, even if the generator selects a sub-optimal rationale, the predictor can still produce a lower prediction loss when given this sub-optimal rationale than when given the optimal rationale that it has never seen. As a result, the generator's selection of the sub-optimal rationale will be reinforced.  In the end, both the rationale generator and the predictor will be trapped in a sub-optimal equilibrium, which hurts both model's predictive accuracy and the quality of generated rationales.   

By investigating the training objective of selective rationalization theoretically, we found that the fundamental cause of the problem of interlocking is that the rationalization objective we aim to minimize is undesirably \emph{concave} with respect to the rationale generator's policy, which leads to many sub-optimal corner solutions.  On the other hand, although the attention-based models (\emph{i.e.}, via soft selection) produce much less faithful explanations and do not have the nice property of certification of exclusion, their optimization objective has a better convexity property with respect to the attention weights under certain assumptions, and thus would not suffer from the interlocking problem.

Motivated by these observations, we propose a new rationalization framework, called \algname (attention-to-rationale), which combines the advantages of both the attention model (convexity) and binarized rationalization (faithfulness) into one.  Specifically, our model consists of a generator, and two predictors. One predictor, called \emph{attention-based predictor}, operates on the soft-attention, and the other predictor, called \emph{binarized predictor}, operates on the binarized rationales.  The attention as used by the attention-based predictor is tied to the rationale selection probability as used by the binarized predictor.  During training,  the generator aims to improve both predictors' performance while minimizing their prediction gap.   As we will show theoretically, the proposed rationalization scheme can overcome the concavity of the original setup, and thus can avoid being trapped in sub-optimal rationales.  In addition, during inference time, we only keep the binarized predictor to ensure the faithfulness of the generated explanations.  We conduct experiments on two synthetic benchmarks and two real datasets.  The results demonstrate that our model can significantly alleviate the problem of interlocking and find explanations that better align with human judgments.

%-------------------------------------------------------------------------
\section{Related Work}
{\bf Selective rationalization: } \cite{lei2016rationalizing} proposes the first generator-predictor framework for rationalization.  Following this work, new game-theoretic frameworks were proposed to encourage different desired properties of the selected rationales, such as optimized Shapley structure scores~\cite{chen2018shapley}, comprehensiveness~\cite{yu2019rethinking},  multi-aspect supports~\cite{antognini2021rationalization, chang2019game} and invariance~\cite{chang2020invariant}.  Another fundamental direction is to overcome the training difficulties. \cite{bao2018deriving} replaces policy gradient with Gumbel softmax.  \cite{yu2019rethinking} proposes to first pre-train the predictor,  and then perform end-to-end training.  \cite{chang2019game} adopts different updating dynamics between the generator and the predictor.  \cite{bastings2019interpretable} replaces the Bernoulli sampling distributions with rectified continuous random variables to facilitate constrained optimization.  \cite{sha2020learning} proposes to enhance the training objective with an adversarial information calibration according to a black-box predictor.  However, these methods cannot address the problem of interlocking. 

{\bf Attention as a proxy of explanation:} Model's attention \cite{bahdanau2014neural,kim2017structured,xu2015show} could serve as a proxy of the rationale.  Although attention is easy to obtain, it lacks faithfulness.  An input associated with low attention weight can still significantly impact the prediction.  In addition, recent works \cite{bastings2020elephant, jain2019attention, pruthi2020learning, serrano2019attention,wiegreffe2019attention} also find that the same prediction on an input could be generated by totally different attentions, which limits its applicability to explaining neural predictions.  To improve the faithfulness of attentions, \cite{mohankumar2020towards,tutek2020staying} regularize the hidden representations on which the attention is computed over;  \cite{glockner2020you} applies attention weights on losses of pre-defined individual rationale candidates' predictions. Nevertheless, rationales remain to be more faithful explanations due to their certification of exclusion.

\cite{guerreiro2021spectra,treviso2020explanation} force the sparsity of the attention with sparsemax~\cite{martins2016softmax}, so as to promote the faithfulness of their attention as rationales.
The interlocking problem still persists in this framework, because the loss landscape remains concave (refer to our arguments in Section~\ref{subsec:toy_example}\&\ref{subsec:concavity_of_rationalization}).
Specifically, since the predictor would not see the sentences that receive 0 attention weights, it tends to underfit these sentences. As a result, the generator does not have the incentive to assign positive weights to the sentences that are previously assigned zero weights, thus is prone to selecting the same sentences.

{\bf Model interpretability beyond selective rationalization: } There are other popular interpretability frameworks besides selective rationalization. Module networks \cite{andreas2016learning,andreas2016neural,johnson2017inferring} compose appropriate neural modules following a logical program to complete the task. Their applicability is relatively limited, due to the requirement of pre-defined modules and programs.   Evaluating feature importance with gradient information \cite{bastings2020elephant,li2016visualizing,simonyan2013deep, sundararajan2017axiomatic} is another popular method.  Though \cite{bastings2020elephant} discusses several advantages of gradient-based methods over rationalization, they are post-hoc and cannot impose structural constraints on the explanation.  Other lines of work that provide post-hoc explanations include local perturbations \cite{kononenko2010efficient, lundberg2017unified}; locally fitting interpretable models~\cite{alvarez2018towards,ribeiro2016should}; and generating explanations in the form of edits to inputs that change model prediction to the contrast case~\cite{ross2020explaining}.

%-------------------------------------------------------------------------
\section{Selective Rationalization and Interlocking}
\label{sec:analysis}

\begin{figure}[t!]
    \centering
    \includegraphics[width=0.95\linewidth]{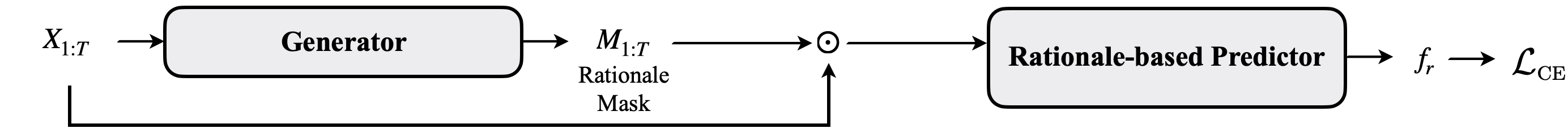}
    \caption{\small{A conventional selective rationalization framework.}}
    \label{fig:baseline}
\end{figure}

In this section, we will formally analyze the problem of interlocking in conventional selective rationalization frameworks. Throughout this section, upper-cased letters, \emph{i.e.}, \e{\bm A} and \e{A}, represent random vectors (bolded) and random values (unbolded) respectively; lower cased letters, \emph{i.e.}, $\bm a$ and $a$, represent deterministic vectors (bolded) and values (unbolded) respectively. 
Vectors with a colon subscript, \emph{i.e.}, \e{\bm a_{1:T}}, represent a concatenation of \e{\bm a_1} to \e{\bm a_T}, \emph{i.e.}, \e{[\bm a_1; \cdots; \bm a_T]}.

\subsection{Overview of Selective Rationalization}
\label{subsec:rationale_overview}

Consider a classification problem, \e{(\bm X, Y)}, where \e{\bm X = \bm X_{1:T}} is the input feature, and $Y$ is the discrete class label. In NLP applications, \e{\bm X_{1:T}} can be understood as a series of $T$ words/sentences. The goal of selective rationalization is to identify a \emph{binary} mask, \e{\bm M \in \{0, 1\}^T}, that applies to the input features to form a rationale vector, \e{\bm Z}, as an explanation of \e{Y}. Formally, the rationale vector \e{\bm Z} is defined as
\begin{equation}
\small
    \bm Z = \bm M \circ \bm X \equiv [M_1 \bm X_1, \cdots, M_T \bm X_T].
    \label{eq:rationale_def}
\end{equation}
Conventionally, \e{\bm Z} is determined by maximizing the mutual information between \e{\bm Z} and \e{Y}, \emph{i.e.},
\begin{equation}
\small
    \max_{\bm M} I(Y; \bm M \circ \bm X), \quad \mbox{s.t. }\bm M \in \mathcal{M},
    \label{eq:mmi}
\end{equation}
where \e{\mathcal{M}} refers to a constraint set, such as the sparsity constraint and a continuity constraint, requiring that the selected rationale should be a small and continuous subset of the input features. 

One way of learning to extract the rationale under this criterion is to introduce a game-theoretic framework (see Figure~\ref{fig:baseline}) consisting of two players, a rationale generator and a predictor.  The rationale generator selects a subset of input as rationales and the predictor makes the prediction based only on the rationales. The two players cooperate to maximize the prediction accuracy, so the rationale generator would need to select the most informative input subset.

Specifically, the rationale generator generates a probability distribution, \e{\bm \pi}, for the masks, based on the input features \e{\bm X}. Then, the mask \e{\bm M} is randomly drawn from the distribution \e{\bm \pi}. To simplify our exposition, we focus on the case that \e{\bm X_i} represents a sentence and only one of the \e{T} sentences is selected as a rationale. In this case, \e{\bm M} is a one-hot vector, and \e{\bm \pi} is a multinomial distribution.  Formally, the mask \e{\bm M} is generated as follows
\begin{equation}
    \small
    \bm M \sim \bm \pi(\bm X) = [\pi_1(\bm X), \cdots, \pi_T(\bm X)], \quad \mbox{where } \pi_i(\bm X) =  p(\bm M = \bm e_i | \bm X)], \nonumber
\end{equation}
and \e{\bm e_i} denotes a \e{T}-dimensional one-hot vector, with the $i$-th dimension equal to one. 
The generalization to making multiple selections will be discussed in Section~\ref{ssec:method}.

After the mask is generated, the predictor, \e{\bm f_r(\cdot; \bm \theta_r)} (the subscript \e{r} stands for rationale to differentiate from the attention-based predictor introduced later), then predicts the probability distribution of \e{Y} based only on \e{\bm Z = \bm M \circ \bm X}, \emph{i.e.},
\begin{equation}
    \small
    \bm f_r(\bm Z; \bm \theta_r) = [\hat{p}(Y = 1 | \bm Z), \cdots, \hat{p}(Y=c | \bm Z)],
    \label{eq:hard_predictor_def}
\end{equation}
where \e{\hat{p}} represents a predicted distribution, and \e{\bm \theta_r} denotes the parameters of the predictor.

The generator and the predictor are trained jointly to minimize the cross-entropy loss of the prediction:
\begin{equation}
    \small
    \min_{\bm \pi(\cdot), \bm \theta_r} \mathcal{L}_r(\bm \pi, \bm \theta_r), \quad \mbox{where } \mathcal{L}_r(\bm \pi, \bm \theta_r)  = \mathbb{E}_{\substack{\bm X, Y \sim \mathcal{D}_{tr} \\ \bm M \sim \bm \pi(\bm X)}} [\ell(Y, \bm f_r(\bm M \circ \bm X; \bm \theta_r))].
    \label{eq:rat_prob}
\end{equation}
\e{\mathcal{D}_{tr}} denotes the training set; \e{\ell(\cdot, \cdot)} denotes the cross entropy loss. 
It can be shown \cite{chen2018learning} that, if \e{\bm \pi(\cdot)} and \e{\bm f_r(\cdot; \bm \theta_r)} both have sufficient representation power, the globally optimal \e{\bm \pi(\bm X)} of Equation~\eqref{eq:rat_prob} would generate masks \e{\bm M} that are globally optimal under Equation~\eqref{eq:mmi}.

% =====================================

\subsection{Interlocking: A Toy Example}
\label{subsec:toy_example}

Despite the nice guarantee of its global optimum solution, the rationalization framework in Equation~\eqref{eq:rat_prob} suffers from the problem of being easily trapped into poor local minima, a problem we refer to as \emph{interlocking}. To help readers understand the nature of this problem, we would like to start with a toy example, where the input consists of two sentences, \e{\bm X_1} and \e{\bm X_2}. We assume that \e{\bm X_1} is the more informative (in terms of predicting \e{Y}) sentence between the two, so the optimal solution for the rationale generator \e{\bm \pi} is to always select \e{\bm X_1} (\emph{i.e.} \e{\pi_1=1}, and \e{\pi_2 = 0}).

However, assume, for some reason, that the generator is initialized so poorly that it only selects \e{\bm X_2}, and that the predictor has been trained to make the prediction based only on \e{\bm X_2}. In this case, we will show that it is very hard for the generator-predictor to escape from this poor local minimum thus it fails to converge to the globally optimal solution of selecting \e{\bm X_1}. Since the predictor underfits to \e{\bm X_1}, it will produce a large prediction error when \e{\bm X_1} is fed. As a result, the rationale generator would stick with selecting \e{\bm X_2} because \e{\bm X_2} yields a smaller prediction error than \e{\bm X_1}. The predictor, in turn, would keep overfitting to \e{\bm X_2} and underfitting to \e{\bm X_1}. In short, both players lock the other player from escaping from the poor solution, hence the name interlocking.

\begin{wraptable}{r}{0.53\textwidth}
\vspace{-0.1in}
    \small
    \centering
    \caption{\small{An example payoff (negative loss) table of the accordance game between the generator (Gen) and the predictor (Pred), where the interlocking problem is manifested as multiple Nash Equilibria.}}
    \label{tab:payoff}
    % \vspace{-0.1in}
    \begin{tabular}{cc||cc}
        \hline
        &&\multicolumn{2}{c}{\textbf{Pred.}} \\
        && \textit{Overfit to} \e{\bm X_1} & \textit{Overfit to} \e{\bm X_2} \\
        \hline\hline
        \multirow{2}{*}{\textbf{Gen.}} & \textit{Select} $\bm X_1$ & $(-1, -1)$ & $(-10, -10)$  \\
         & \textit{Select} \e{\bm X_2} & $(-20, -20)$ & $(-2, -2)$ \\
         \hline
    \end{tabular}
    \vspace{-0.05in}
\end{wraptable}
% \endgroup
The problem of interlocking can also be manifested by an accordance game, where the generator has two strategies, \emph{select} \e{\bm X_1} and \emph{select} \e{\bm X_2}, and the predictor also has two strategies, \emph{overfit to} \e{\bm X_1} and \emph{overfit to} \e{\bm X_2}. An example payoff table is shown in Table~\ref{tab:payoff}. As can be seen, (\emph{select} \e{\bm X_1}, \emph{overfit to} \e{\bm X_1}) has the highest payoff, and thus is the optimal solution for both players. However, (\emph{select} \e{\bm X_2}, \emph{overfit to} \e{\bm X_2}) also constitutes a Nash equilibrium, which is locally optimal.

\subsection{Interlocking and Concave Minimization}
\label{subsec:concavity_of_rationalization}

To understand the fundamental cause of the interlocking problem, rewrite the optimization problem in Equation~\eqref{eq:rat_prob} into a nested form:
\begin{equation}
\small
    \min_{\bm \pi(\cdot), \bm \theta_r} \mathcal{L}_r(\bm \pi, \bm \theta_r) = \min_{\bm \pi(\cdot)} \min_{\bm \theta_r} \mathcal{L}_r(\bm \pi, \bm \theta_r) = \min_{\bm \pi(\cdot)} \mathcal{L}_r(\bm \pi, \bm \theta_r^*(\bm \pi)),
    \label{eq:nested_obj}
\end{equation}
\begin{equation}
\small
    \mbox{where}\quad \bm \theta_r^*(\bm \pi) = \argmin_{\bm \theta_r} \mathcal{L}_r(\bm \pi, \bm \theta_r).
    \label{eq:f*_def}
\end{equation}
Furthermore, denote
\begin{equation}
    \small
    \mathcal{L}_r^*(\bm \pi) = \mathcal{L}_r(\bm \pi, \bm \theta_r^*(\bm \pi)).
    \label{eq:L*_def}
\end{equation}
Then, the problem of finding the optimal rationale boils down to finding the global minimum of \e{\mathcal{L}_r^*(\bm \pi)}. In order to achieve good convergence properties,  \e{\mathcal{L}_r^*(\bm \pi)} would ideally be convex with respect to \e{\bm \pi}. However, the following theorem states the opposite.
\begin{theorem}
    \e{\mathcal{L}_r^*(\bm \pi)} is concave with respect to \e{\bm \pi}.
    \label{thm:concave}
\end{theorem}
The proof is presented in Appendix~\ref{subsec:concave_thm_proof}. Theorem~\ref{thm:concave} implies the cooperative rationalization objective can contain many local optima at the corners. Going back to the two-sentence example, Figure~\ref{fig:convex}\subref{subfig:convex1} plots an example \e{\mathcal{L}_r^*(\bm \pi)} against \e{\pi_1}. Since there are two sentences, \e{\pi_1=0} implies that the generator always selects \e{\bm X_2}, and \e{\pi_1=1} implies the generator always selects \e{\bm X_1}. As shown in the figure, since \e{\bm X_1} is more informative than \e{\bm X_2}, the global minimum is achieved at \e{\pi_1=1}. However, it can be observed that \e{\pi_1=0} is also a local minimum, and therefore the rationalization framework can be undesirably trapped into the rationalization scheme that always selects the worse sentence of the two.

\begin{figure}
    \centering
    \subfloat[]{\includegraphics[width=0.32\linewidth]{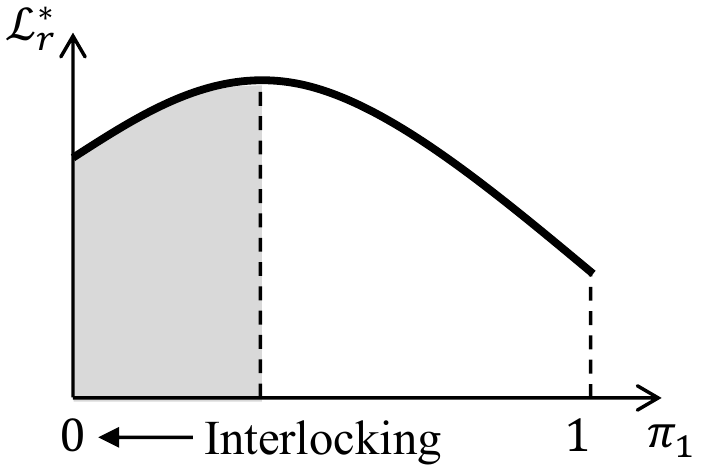}\label{subfig:convex1}}
    \subfloat[]{\includegraphics[width=0.32\linewidth]{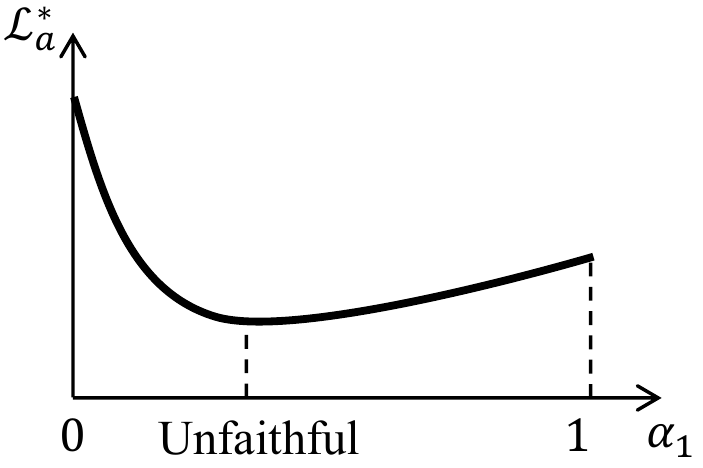}\label{subfig:convex2}}
    \subfloat[]{\includegraphics[width=0.32\linewidth]{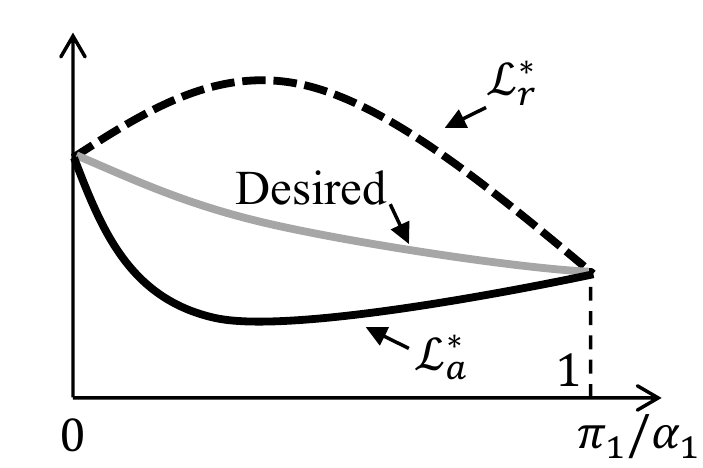}\label{subfig:convex3}}
    \caption{\small{Example loss landscapes of the two-sentence scenario. (a) An example loss landscape of rationale-based explanation (Equation~\eqref{eq:L*_def}), which is concave and induces interlocking dynamics towards a sub-optimal local minimum. (b) An example loss landscape of attention-based explanation (Equation~\eqref{eq:L*_a_def}), which is convex but with an unfaithful global minimum. (c) The two loss landscapes share common end points. Desirable landscapes should lie in between.}}
    \label{fig:convex}
\end{figure}

\subsection{Convexity of Attention-based Explanation}
\label{subsec:attention}

Knowing that the selective rationalization has an undesirable concave objective, 
we now turn to another class of explanation scheme, \emph{i.e.}, attention-based explanation, which uses soft attention, rather than binary selection, of the input as an explanation. Specifically, we would like to investigate whether its objective has a more or less desirable convexity property than that of selective rationalization.

Formally, consider an attention-based predictor, \e{\bm f_a(\bm \alpha(\bm X) \odot \bm X; \bm \theta_a)} (the subscript \e{a} stands for attention), which is almost identical to the rationalization predictor in Equation~\eqref{eq:hard_predictor_def}, except that the binary mask \e{\bm M} is replaced with a soft attention weight \e{\bm \alpha(\bm X)} where each dimension sums to one. So the optimization objective becomes
\begin{equation}
    \small
    \min_{\bm \alpha(\cdot), \bm \theta_a} \mathcal{L}_a(\bm \alpha, \bm \theta_a), \quad \mbox{where } \mathcal{L}_a(\bm \alpha, \bm \theta_a) = \mathbb{E}_{\bm X, Y \sim D_{tr}}[\ell(Y, \bm f_a(\bm \alpha(\bm X) \odot \bm X; \bm \theta_a))].
    \label{eq:att_prob}
\end{equation}
Similar to Equations~\eqref{eq:nested_obj} to \eqref{eq:L*_def}, define
\begin{equation}
    \small
    \mathcal{L}_a^*(\bm \alpha) = \mathcal{L}(\bm \alpha, \bm \theta_a^*(\bm \alpha)), \quad \mbox{where } \bm \theta_a^*(\bm \alpha) = \argmin_{\bm \theta_a} \mathcal{L}_a(\bm \alpha, \bm \theta_a).
    \label{eq:L*_a_def}
\end{equation}
The following theorem shows that \e{\mathcal{L}_a^*(\bm \alpha)} has a more desirable convexity property.

\begin{theorem}
    \e{\mathcal{L}_a^*(\bm \alpha)} is convex with respect to \e{\bm \alpha}, if
    
    \quad 1. \e{\mathcal{L}_a(\bm \alpha, \bm \theta_a)} is \e{\mu}-strongly convex with respect to \e{\bm \alpha} with \e{\ell_2} distance metric, \e{\forall \bm \theta_a}; 
    % \label{assump:strong_convex}
    
    \quad 2. \e{\mathcal{L}_a(\bm \alpha, \bm \theta_a^*(\bm \alpha'))} has a bounded regret with the optimal loss, i.e., when \e{\bm \alpha'=\bm \alpha}, with \e{\ell_2} norm:
        \begin{equation}
        \small
          \mathcal{L}_a(\bm \alpha, \bm \theta_a^*(\bm \alpha')) - \mathcal{L}_a(\bm \alpha, \bm \theta_a^*(\bm \alpha)) \leq \frac{l}{2} \mathbb{E}\left[\lVert\bm \alpha(\bm X) - \bm \alpha'(\bm X)\rVert_2\right]^2, \quad \forall \bm \alpha(\cdot), \bm \alpha'(\cdot);
        \end{equation} 
        % \label{assump:bounded_regret}
    \quad 3. \e{\mu \geq l}.
    \label{thm:convexity}
\end{theorem}
The proof is presented in Appendix~\ref{subsec:convex_thm_proof}, where we also discuss the feasibility of the assumptions.
A special case where the predictor has sufficient representation power is discussed in Appendix~\ref{subsec:special_case}
Figure~\ref{fig:convex}\subref{subfig:convex2} plots an example \e{\mathcal{L}_a(\bm \alpha)} against \e{\alpha_1}, again under the same two-sentence toy scenario. Note that \e{\alpha_1=0} means \e{\bm X_2} gets all the weight; \e{\alpha_1=1} means \e{\bm X_1} gets all the weights. As can be observed, \e{\mathcal{L}_a(\bm \alpha)} is now a convex function, which makes it more desirable in terms of optimization. However, the example in Figure~\ref{fig:convex}\subref{subfig:convex2} also shows why such attention-based scheme is sometimes not faithful. Even though \e{\bm X_1} is a better sentence of the two, the global minimum \e{\mathcal{L}_a(\bm \alpha)} is achieved at the point where \e{\bm X_2} gets a larger weight than \e{\bm X_1} does.  The reason why the global minimum is usually achieved in the interior (\e{0 <\alpha_1 < 1}) rather than the corner (\e{\alpha_1 = 0} or \e{1}) is that the predictor would have access to more information if both \e{\bm X_1} and \e{\bm X_2} get non-zero attention weights. 

\subsection{Comparing Binary Selection and Soft Attention}

Figure~\ref{fig:convex}\subref{subfig:convex3} puts together the two loss landscapes, \e{\mathcal{L}^*(\bm \pi)} and \e{\mathcal{L}_a^*(\bm \alpha)}, with the rationale selection probability tied to the attention weights, \emph{i.e.}, \e{\bm \pi = \bm \alpha}. There are two important observations. First, the two loss functions take the same values at the two corners, \e{\pi_1=\alpha_1=0} and \e{\pi_1=\alpha_1=1}, because at either corner case, both binary selection and soft attention schemes would exclusively select one of the two sentences, hence yielding the same loss, if both predictors have the same architecture and parameterization. Second, the binary selection and soft attention have complementary advantages. The former has a faithful global minimum but concave; the latter is convex but the global minimum is not faithful. Therefore, both advantages can be simultaneously achieved if we can design a system with a loss landscape that lies in between the two loss functions, as shown by the gray curve.

%-------------------------------------------------------------------------
\section{The Proposed \algname (Attention-to-Rationale) Framework}
\label{sec:method}

\begin{figure}[t!]
    \centering
    \includegraphics[width=0.95\linewidth]{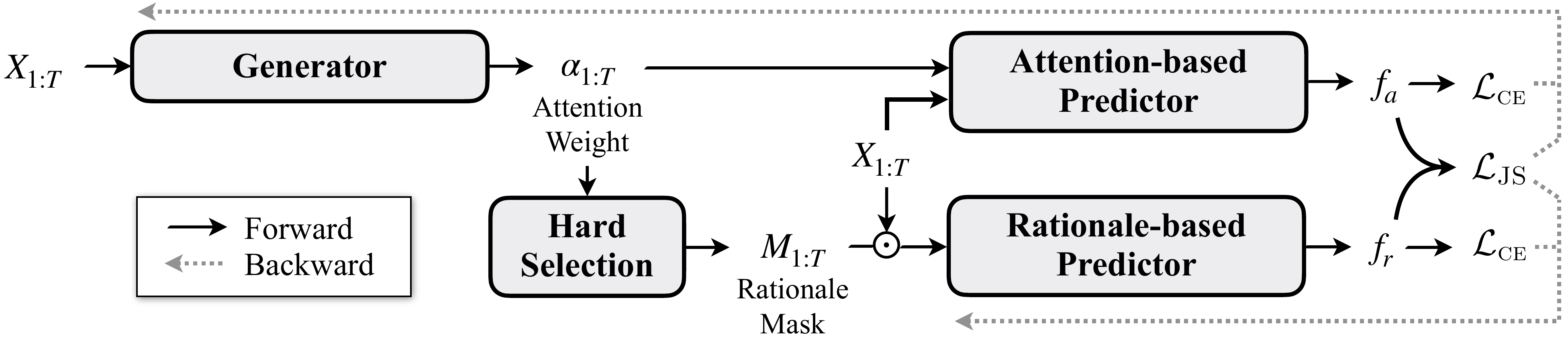}
    \caption{\small{Our proposed rationalization architecture.}}
    \label{fig:method}
\end{figure}

\subsection{The \algname Architecture}
\label{ssec:method}

Our proposed \algname aims to 
combine the merits of selective rationalization and attention-based explanations. Figure~\ref{fig:method} shows the architecture of \algnamens. \algname 
consists of three modules, a \emph{rationale generator}, a \emph{rationale-based predictor}, and an \emph{attention-based predictor}.

The \emph{rationale generator} generates a soft attention, \e{\bm \alpha(\bm X)}.
The same soft attention also serves as the probability distribution from which the rationale selection mask, \e{\bm M}, is drawn. \emph{i.e.}, \e{\bm M \sim \bm \alpha(\bm X)}.
The \emph{rationale-based predictor}, \e{\bm f_r(\cdot; \bm \theta_r)}, predicts the output \e{Y} based on the input masked by \e{\bm M}. The \emph{attention-based predictor}, \e{\bm f_a(\cdot; \bm \theta_a)}, predicts the output \e{Y} based on the representation weighted by \e{\bm \alpha(\bm X)}. \e{\bm \theta_r} and \e{\bm \theta_a} denote the parameter of the two predictors, respectively. Formally,
\begin{equation}
\small
    \bm f_r(\bm M \odot \bm X; \bm \theta_r), \quad \bm f_a(\bm X, \bm \alpha(\bm X); \bm \theta_a). \nonumber
\end{equation}
Note that, instead of having the input form of \e{\bm \alpha(\bm X)\odot\bm X} to the attention-based predictor (as in Section~\ref{subsec:attention}), we write \e{\bm X} and \e{\bm \alpha(\bm X)} as two separate inputs, to accommodate broader attention mechanisms that weight on the intermediate representations rather than directly on the input.
In the experiments, we implement this general framework following some common practices in the NLP community, with details deferred in Section~\ref{ssec:exp_details}.

It is worth emphasizing that the output of the rationale generator, \e{\bm \alpha(\bm X)}, is just one set of attention weights, but has two uses. First, it is used to directly weight the input features, which is fed to the attention-based predictor. Second, it is used to characterize the distribution of the rationale mask \e{\bm M}. The rationale mask is applied to the input feature, which is then fed to the rationale-based predictor. 

So far, our discussion has focused on the case where only one of the input features is selected as the rationale. 
\algname can generalize to the case where multiple input features are selected. In this case, the rationale mask \e{\bm M} can have multiple dimensions equal to one. In our implementation, \e{\bm M} is determined by retaining \e{q\%} largest elements of \e{\bm \alpha(\bm X)}, where \e{q} is a preset sparsity level.  

\subsection{The Training Objectives}

The three components have slightly different training objectives. The rationale-based predictor minimizes its prediction loss, while reducing the gap between the two predictors, \emph{i.e.}
\begin{equation}
    \small
    \min_{\bm \theta_r}  \mathcal{L}_r(\bm \pi, \bm \theta_r) + \lambda \mathcal{L}_{JS}(\bm \pi, \bm \theta_r, \bm \theta_a),
    \label{eq:a2r_rationale_obj}
\end{equation} 
where \e{\mathcal{L}_r(\bm \pi, \bm \theta_r)} is the prediction loss of the rationale-based predictor defined in Equation~\eqref{eq:rat_prob}. \e{\mathcal{L}_{JS}(\bm \pi, \bm \theta_r, \bm \theta_a)} is the Jensen-Shannon divergence between the two predicted distributions, defined as
\begin{equation}
\small
    \mathcal{L}_{JS}(\bm \pi, \bm \theta_r, \bm \theta_a) = \mathbb{E}_{\substack{\bm X \sim \mathcal{D}_{tr} \\ \bm M \sim \bm \alpha(\bm X)}} \left[ JS(\bm f_r(\bm M \odot \bm X; \bm \theta_r) \Vert \bm f_a(\bm X, \bm \alpha(\bm X); \bm \theta_a))\right]. \nonumber
\end{equation}
We select the JS divergence because it matches the scale and gradient behavior of the other loss terms.

Both the rationale generator and the attention-based predictor try to minimize the prediction loss of the attention-based predictor, while again reducing the gap between the two predictors, \emph{i.e.},
\begin{equation}
\small
    \min_{\bm \pi(\cdot), \bm \theta_a} \mathcal{L}_a(\bm \pi, \bm \theta_a) + \lambda \mathcal{L}_{JS}(\bm \pi, \bm \theta_r, \bm \theta_a),
    \label{eq:a2r_attention_obj}
\end{equation}
where \e{\mathcal{L}_a(\bm \pi, \bm \theta_a)} is the prediction loss of the attention-based predictor defined in Equation~\eqref{eq:att_prob}. Note that both Equation~\eqref{eq:a2r_rationale_obj} and \eqref{eq:a2r_attention_obj} can be optimized using standard gradient-descent-based techniques. The gradient of the rationale-based predictor does not prapagate back to the generator.  

\subsection{How Does \algname Work}

Essentially, \algname constructs a loss landscape that lies between those of the rationale-based predictor and the attention-based predictor. To better show this, we would like to return to the toy scenario illustrated in Figure~\ref{fig:convex}\subref{subfig:convex3}. If the \e{\lambda} in Equation~\eqref{eq:a2r_attention_obj} is zero, then the loss for the rationale generator would be exactly the lowest  curve (\emph{i.e.}, \e{\mathcal{L}_a^*}).
As \e{\lambda} increases, the attention-based loss curve would shift upward towards the rationale-based loss. As a result, the actual loss curve for the generator will resemble the gray curve in the middle, which addresses the concavity problem and thus the interlocking problem, without introducing unfaithful solutions.  
We use only the attention-based predictor to govern the generator, rather than passing the gradient of both predictors to the generator, because the gradient of \e{\mathcal{L}_a} is much more stable than that of \e{\mathcal{L}_r}, which involves the policy gradient.

%-------------------------------------------------------------------------
\section{Experiments}

\subsection{Datasets}
\label{ssec:datasets}
Two datasets are used in our experiments.  Table~\ref{tab:dataset_stats} in Appendix~\ref{app:data_stats} shows their statistics. Both datasets contain human annotations, which facilitate automatic evaluation of the rationale quality.
To our best knowledge, neither dataset contains personally identifiable information or offensive content.

{\bf{\emph{BeerAdvocate}}: }
BeerAdvocate from \cite{mcauley2012learning} is a multi-aspect sentiment prediction dataset, which has been commonly used in the field of rationalization \cite{bao2018deriving, chang2019game, lei2016rationalizing, yu2019rethinking}.  This dataset includes sentence-level annotations, where each sentence is annotated with one or multiple aspect labels.

{\bf{\emph{MovieReview}}: }
The \emph{MovieReview} dataset is from the \emph{Eraser} benchmark \cite{deyoung2019eraser}.  MovieReview is a sentiment prediction dataset that contains phrase-level rationale annotations.

% =================================================

\subsection{Baselines and Implementation Details}
\label{ssec:exp_details}

We compare to the original rationalization technique \textsc{Rnp} \cite{lei2016rationalizing}, and several published models that achieve state-of-the-art results on real-world benchmarks, which include \textsc{3Player} \cite{yu2019rethinking}, \textsc{HardKuma}\footnote{\url{https://github.com/bastings/interpretable_predictions}.} \cite{bastings2019interpretable}, and  \textsc{BERT-Rnp}~\cite{deyoung2019eraser}.  \textsc{3Player} model builds upon the original \textsc{Rnp} and encourages the completeness of rationale selection.  \textsc{HardKuma} is a token-level method that optimizes the dependent selection of \textsc{Rnp} to encourage more human-interpretable extractions.  \textsc{BERT-Rnp} re-implements the original \textsc{Rnp} with more powerful BERT generator and predictor.  
\textsc{Rnp} is our main baseline to directly compare with, as \textsc{Rnp} and our \algname match in granularity of selection, optimization algorithm and model architecture. We include the other baselines to show the competitiveness of our \algname.

We follow the commonly used rationalization architectures~\cite{bastings2019interpretable,lei2016rationalizing} in our implementations:
We use bidirectional gated recurrent units (GRU) \cite{chung2014empirical} in the generators and the predictors for both our \algname and our reimplemented \textsc{Rnp}. For \algnamens, we share the parameters of both predictors' GRU while leaving the output layers' parameters separated.
Our rationale predictor \e{\bm f_r} encodes the masked input \e{\bm M \odot \bm X} into the hidden states, followed by max-pooling.
The attention-based predictor \e{\bm f_a} encodes the entire input \e{\bm X} into hidden states, which is then weighted by \e{\bm \alpha}.

All methods are initialized with 100-dimension Glove embeddings~\cite{pennington2014glove}.  The hidden state dimensions is 200 for BeerAdvocate, and 100 for MovieReview.  We use Adam  \cite{kingma2014adam} as the default optimizer with a learning rate of 0.001.  The policy gradient update uses a learning rate of 1e-4. The exploration rate is 0.2.
The aforementioned hyperparameters and the best models to report are selected according to the development set accuracy.
Every compared model is trained on a single V100 GPU.

% =================================================

\subsection{Synthetic Experiments}
\label{ssec:exp_synthetic}

To better evaluate the interlocking dynamics, we first conduct two synthetic experiments using the BeerAdvocate dataset, where we deliberately induce interlocking dynamics.  We compare our \algname with \textsc{Rnp}, which is closest to our analyzed framework in Section~\ref{sec:analysis} that suffers from interlocking.

\begin{table}[t!]
\small
\centering
% \fontsize{7}{10.5}\selectfont
    \begin{tabular}{ll||ccccc||ccccc}
    \parbox[t]{7mm}{\multirow{2}{*}{Aspect}}&\multirow{2}{*}{Setting}& \multicolumn{5}{c||}{\textsc{Rnp}} &\multicolumn{5}{c}{\algnamens}\\
    && Acc & P & R & F1 & $X_1$\%& Acc & P & R & F1 & $X_1$\%\\
    \midrule\midrule
    \parbox[t]{2mm}{\multirow{3}{*}{\rotatebox[origin=c]{30}{Aroma}}} & Skew10 & 82.6 & 68.5 & 63.7 & 61.5 & 14.5 & 84.5 & 78.3 & 70.6 & \bf 69.2 & 10.4 \\
    &Skew15 & 80.4 & 54.5 & 51.6 & 49.3 & 31.2 & 81.8 & 58.1 & 53.3 & \bf 51.7 & 35.7 \\
    &Skew20 &  76.8 &  10.8 & 14.1 & 11.0 & 80.5 & 80.0 & 51.7 &  47.9 & \bf 46.3 & 41.5 \\
    \midrule\midrule
    \parbox[t]{2mm}{\multirow{3}{*}{\rotatebox[origin=c]{30}{Palate}}}&Skew10 & 77.3 & 5.6 & 7.4 & 5.5 & 63.9 & 82.8 & 50.3 & 48.0 & \bf 45.5 & 27.5 \\
    &Skew15 & 77.1 & 1.2 & 2.5 & 1.3 & 83.1 &  80.9 & 30.2 & 29.9 & \bf 27.7 &  58.0  \\
    &Skew20 &  75.6 &  0.4 & 1.4 & \bf 0.6 & 100.0 &  76.7 & 0.4 & 1.6 & \bf 0.6 & 97.0  \\
    \midrule\midrule
    \parbox[t]{2mm}{\multirow{3}{*}{\rotatebox[origin=c]{30}{Aroma}}}&Biased0.7 & 84.7 & 71.0 & 65.4 & 63.4 & 12.6  & 85.5 & 77.9 &  70.4 & \bf  69.0 &  12.2\\
    &Biased0.75 & 84.4 & 58.1 & 54.5 & 52.3 & 25.3 & 85.3 & 68.4 & 61.7 & \bf  60.5 &  20.9  \\
    &Biased0.8 & 83.3 & 2.6 & 6.0 & 3.4 & 99.9 & 85.8 & 59.7 &  54.8 & \bf  53.2 &  29.8\\
    \midrule\midrule
    \parbox[t]{7mm}{\multirow{3}{*}{\rotatebox[origin=c]{30}{Palate}}}&Biased0.7 & 83.9 & 51.4 & 50.5 & 47.3 & 24.3  &  83.5 & 55.0 &  52.9 &  \bf 50.1 &  18.8 \\
    &Biased0.75 & 80.0 & 0.4 & 1.4 & 0.6 & 100.0 &  82.8 & 52.7 &  50.7 & \bf  47.9 &  22.0  \\
    &Biased0.8 & 82.0 & 0.4 & 1.4 & 0.6  & 100.0 &   83.6 & 47.9 &  46.2 & \bf  43.5 &  29.6 \\
    \end{tabular}
    \caption{\small{Results on Beer-Skew (top) and Beer-Biased (bottom). 
    P, R, and F1 indicate the token-level precision, recall, and F1 of rationale selection.
    $X_1$\% refers to the ratio of first sentence selection (lower is better). 
    The aroma and palate aspects have 0.5\% and 0.2\% of the testing examples with groundtruth rationales located in the first sentence, respectively.
    \textbf{Bold} numbers refer to the better performance between \textsc{Rnp} and \algname in each setting.
    }}
    \label{tab:sent_level_skewed}
\end{table}

{\bf Beer-Skewed: }
In the first synthetic experiment, we let the rationale predictor overfit the first sentence of each example at the initialization. In the BeerAdvocate dataset, the first sentence is usually about the appearance aspect of the beer, and thus is rarely the optimal rationale when the explanation target is the sentiment for the aroma or palate aspects.  However, by pre-training rationale predictor on the first sentence, we expect to induce an interlocking dynamics toward selecting the sub-optimal first sentence. Specifically, we pre-train the rationale predictor for $k$ epochs by only feeding the first sentence.  Once pre-trained, we then initialize the generator and train the entire rationalization pipeline.  We set $k$ to be 10, 15, and 20 for our experiments.

Table~\ref{tab:sent_level_skewed} (top) shows the result in the synthetic Beer-Skewed setting. 
The $k$ in `Skew$k$' denotes the number of pre-training epochs. The larger the $k$, the more serious the overfitting.
$X_1\%$ denotes the percentage of the test examples where the first sentence is selected as rationale. 
The higher $X_1\%$ is, the worse the algorithm suffers from interlocking.
There are two important observations. 
First, when the number of skewed training epochs increases, the model performance becomes worse, \emph{i.e.}, it becomes harder for the models to escape from interlocking.
Second, 
the \textsc{Rnp} model fails to escape in the Aroma-Skew20 setting and all the palate settings (in terms of low F1 scores), while our \algname can rescue the training process except for Palate-Skew20. For the other settings, both models can switch to better selection modes but the performance gaps between the \textsc{Rnp} and our methods are large.

We further study the failure in the Palate-Skew20 setting with another experiment where we set $\lambda$=0 to degrade our system a soft-attention system, which in theory would not suffer from interlocking. In the mean time it still generates the hard mask as rationales and trains the rationale-based predictor.
This results in a 2.2\% F1 score, with 97.3\% $X_1$ selection -- \emph{i.e.}, the soft model also fails.
This suggests that the failure of \algname may not be ascribed to its inability to cope with interlocking, but possibly to the gradient saturation of the predictor.

{\bf Beer-Biased: }
The second setup considers interlocking caused by strong spurious correlations.   We follow a similar setup in \cite{chang2020invariant} 
to append punctuation ``,'' and ``.'' at the beginning of the first sentence with the following distributions: 
\begin{equation*}
\small
p(\text{append , } | Y=1) = p(\text{append . } | Y=0) = \alpha; ~~
p(\text{append . } | Y=1) = p(\text{append , } | Y=0) = 1 - \alpha.
\label{eq:pollution}
\end{equation*}
We set $\alpha$ to 0.7, 0.75, and 0.8 for our experiments, which are all below the achievable accuracy that selecting the true rationales.
Intuitively, since sentence one now contains the appended punctuation, which is an easy-to-capture clue, we expect to induce an interlocking dynamics towards selecting the first sentence, even though the appended punctuation is not as predictive as the true rationales.

Table~\ref{tab:sent_level_skewed} (bottom) shows the result in the synthetic Beer-Biased setting. The result is similar to that in the Beer-Skewed setting.
First, the higher correlated bias makes it more difficult for the models to escape from interlocking.
Second, our model can significantly outperforms the baseline across all the settings.
Third, the \textsc{Rnp} model fails to escape in the Aroma-Biased0.8 and the Palate-Biased settings with biases ratios of 0.75 and 0.8, while our \algname can do well for all of them.

\begingroup
\setlength{\tabcolsep}{5.5pt}
\begin{table}[t!]
    \caption{\small{Full results on Beer Review. Our \algname achieves best results on all the aspects. Note that the appearance aspect does not suffer from interlocking so all approaches performs similarly.}}
    \label{tab:sent_level_results}
    \centering
    \small
    \begin{tabular}{l||cccc||cccc||cccc}
    & \multicolumn{4}{c||}{Appearance} & \multicolumn{4}{c||}{Aroma} & \multicolumn{4}{c}{Palate} \\
    & Acc & P & R & F1& Acc & P & R & F1& Acc & P & R & F1 \\
    \midrule\midrule
    HardKuma~\cite{bastings2019interpretable} & 86.0 &81.0 & 69.9 & 71.5& 85.7&74.0 &  72.4 & 68.1&  84.4 &45.4 & 73.0 & 46.7 \\
    \textsc{Rnp} & 85.7 & 83.9 & 71.2 & 72.8  & 84.2 & 73.6 & 67.9 & 65.9 & 83.8 & 55.5 & 54.3 &51.0\\ 
    3\textsc{Player} & 85.8 & 78.3 & 66.9 & 68.2  & 84.6 & 74.8 & 68.5 & 66.7 &  83.9 & 54.9 & 53.5 & 50.3 \\ 
    \midrule\midrule
    Our \algnamens & 86.3 &  84.7  & 71.2  & \bf 72.9 & 84.9 & 79.3  & 71.3  & \bf 70.0 & 84.0  & 64.2  & 60.9  & \bf 58.0 \\
    \quad(std) & $_{\pm 0.2}$  & $_{\pm 1.2}$  & $_{\pm 0.7}$  & $_{\pm 0.8}$   & $_{\pm 0.1}$  & $_{\pm 0.5}$  & $_{\pm 0.3}$  & $_{\pm 0.4}$  & $_{\pm 0.2}$  & $_{\pm 0.7}$  & $_{\pm 0.4}$  & $_{\pm 0.5}$   \\
    \end{tabular}
\end{table}
\endgroup

\subsection{Results on Real-World Settings}
\label{ssec:exp_real_world}

{\bf BeerAdvocate: }
Table~\ref{tab:sent_level_results} gives results on the standard beer review task. Our \algname achieves new state-of-the-art on all the three aspects, in terms of the rationale F1 scores.  All three baselines generate continuous text spans as rationales, thus giving a similar range of performance.  Among them, the state-of-the-art method, HardKuma, is not restricted to selecting a single sentence, but would usually select only 1$\sim$2 long spans as rationales, due to 
the dependent selection model and the strong continuity constraint. Therefore, the method has more freedom in rationale selection compared to the sentence selection in others, and gives high predictive accuracy and good rationalization quality.

\floatsetup[table]{capposition=bottom}
\begin{table*}[t!]
	\small
    % \hspace{-0.31cm}
	\begin{tabular}{p{\linewidth}}
        \emph{BeerAdvocate - Palate Aspect} \\
		\arrayrulecolor{grey}  % choose color
		\midrule
        \hlb{pours}{0} \hlb{a}{0} \hlb{dark}{0} \hlb{brown}{0} \hlb{,}{0} \hlb{almost}{0} \hlb{black}{0} \hlb{color}{0} \hlb{.}{0} \hlb{there}{0} \hlb{is}{0} \hlb{minimal}{0} \hlb{head}{0} \hlb{that}{0} \hlb{goes}{0} \hlb{away}{0} \hlb{almost}{0} \hlb{immediately}{0} \hlb{with}{0} \hlb{only}{0} \hlb{a}{0} \hlb{little}{0} \hlb{lacing}{0} \hlb{.}{0} \hlb{smell}{0} \hlb{is}{0} \hlb{a}{0} \hlb{little}{0} \hlb{subdued}{0} \hlb{.}{0} \hlb{dark}{0} \hlb{coffee}{0} \hlb{malts}{0} \hlb{are}{0} \hlb{the}{0} \hlb{main}{0} \hlb{smell}{0} \hlb{with}{0} \hlb{a}{0} \hlb{slight}{0} \hlb{bit}{0} \hlb{of}{0} \hlb{hops}{0} \hlb{also}{0} \hlb{.}{0} \hlb{taste}{0} \hlb{is}{0} \hlb{mostly}{0} \hlb{of}{0} \hlb{coffee}{0} \hlb{with}{0} \hlb{a}{0} \hlb{little}{0} \hlb{dark}{0} \hlb{chocolate}{0} \hlb{.}{0} \hlb{it}{0} \hlb{starts}{0} \hlb{sweets}{0} \hlb{,}{0} \hlb{but}{0} \hlb{ends}{0} \hlb{with}{0} \hlb{the}{0} \hlb{dry}{0} \hlb{espresso}{0} \hlb{taste}{0} \hlb{.}{0} \hlb{\textbf{\ul{mouthfeel}}}{20} \hlb{\textbf{\ul{is}}}{20} \hlb{\textbf{\ul{thick}}}{20} \hlb{\textbf{\ul{and}}}{20} \hlb{\textbf{\ul{chewy}}}{20} \hlb{\textbf{\ul{like}}}{20} \hlb{\textbf{\ul{a}}}{20} \hlb{\textbf{\ul{stout}}}{20} \hlb{\textbf{\ul{should}}}{20} \hlb{\textbf{\ul{be}}}{20} \hlb{\textbf{\ul{,}}}{20} \hlb{\textbf{\ul{but}}}{20} \hlb{\textbf{\ul{i}}}{20} \hlb{\textbf{\ul{prefer}}}{20} \hlb{\textbf{\ul{a}}}{20} \hlb{\textbf{\ul{smoother}}}{20} \hlb{\textbf{\ul{feel}}}{20} \hlb{\textbf{\ul{.}}}{20} \hlr{\emph{drinkability}}{20} \hlr{\emph{is}}{20} \hlr{\emph{nice}}{20} \hlr{\emph{.}}{20} \hlb{a}{0} \hlb{very}{0} \hlb{good}{0} \hlb{representation}{0} \hlb{for}{0} \hlb{its}{0} \hlb{style}{0} \hlb{.}{0}\\
		% ---------------------------------------------------
	\end{tabular}
	\vspace*{-0.05in}
    \captionof{figure}{\small{Examples of generated rationales on the palate aspect.  Human annotated words are \ul{underlined}. \algname and \textsc{Rnp} rationales are highlighted in \hlb{\textbf{blue}}{20} and \hlr{\emph{red}}{20} colors, respectively.}}
    \label{fig:exp_highlight_short}
\end{table*}
\floatsetup[table]{capposition=top}

\algname achieves a consistent performance advantage over all the baselines on all three aspects. In addition, we have observed evidence suggesting that the performance advantage is likely due to \algnamens's superior handling of the interlocking dynamics. More specifically, most beer reviews contain highly correlated aspects, which can induce interlocking dynamics towards selecting the review of a spuriously correlated aspect, analogous to the appended punctuations in the Beer-Biased synthetic setting. For example, when trained on the aroma or the palate aspect, \textsc{Rnp} has the first 7 epochs selecting the ``overall'' reviews for more than 20\% of the samples.  On the palate aspect, \textsc{Rnp} also selects the aroma reviews for more than 20\% samples in the first 6 epochs. Both of these observations indicate that \textsc{Rnp}  is trapped in a interlocking convergence path. On the appearance aspect, we do not observe severe interlocking trajectories in \textsc{Rnp}; therefore for this aspect, we do not expect a huge improvement in our proposed algorithm.
The aforementioned training dynamics explain why our approach has a larger performance advantage on aroma and palate aspects (4.5\% and 7.4\% in F1 respectively) than on appearance. Figure~\ref{fig:exp_highlight_short} gives an example where the \textsc{Rnp} makes a mistake of selecting the ``overall'' review.  More examples can be found in Appendix~\ref{app:examples}.

{\bf MovieReview: }
Table~\ref{tab:token_level_movie} gives results on the movie review task.  Since the human rationales are multiple phrase pieces, we make both \textsc{Rnp} and \algname perform token-level selections to better fits to this task. We follow the standard setting \cite{bao2018deriving,lei2016rationalizing} to use the sparsity and continuity constraints to regularize the selected rationales for all methods.  For fair comparisons, we use a strong constraint weight of 1.0 to punish all algorithms that highlight more than 20\% of the inputs, or have more than 10 isolated spans.  These numbers are selected according to the statistics of the rationale annotations.

\begin{wraptable}{r}{0.45\textwidth}
\caption{\small{Results on movie review.}}
\small
\centering
\begin{tabular}{l||ccc}
 & P & R & F1    \\
    \midrule\midrule
    \textsc{Rnp} impl by \cite{lehman2019inferring} &--&--&13.9\\
    BERT-\textsc{Rnp} \cite{deyoung2019eraser} &--&--&32.2\\
    \textsc{HardKuma}~\cite{bastings2019interpretable} & 31.1 & 28.3 & 27.0\\
    \textsc{Rnp}&35.6 & 21.1 & 24.1\\
    \textsc{3Player}  & 38.2 & 26.0 & 28.0 \\ 
    \midrule\midrule
    Our \algname & \bf 48.7 & \bf 31.9 & \bf 34.9$_{\pm 0.5}$ \\ 
    \end{tabular}
    \label{tab:token_level_movie}
    \vspace{-0.1in}
\end{wraptable}

Different from BeerAdvocate,  the annotations of MovieReview are at the phrase-level, which are formed as multiple short spans.   In addition, these annotated rationales often tend to be ``over-complete'', \emph{i.e.,} they contain multiple phrases, all of which are individually highly predictive of the output.  Because of this, the advantage of \textsc{HardKuma} becomes less obvious compared to other baselines.   Yet it still outperforms two different implementations of \textsc{Rnp} (\emph{i.e.,} the published result in ~\cite{lehman2019inferring}, and our own implementation).  Our \algname method consistently beats all the baselines including the strong BERT-based approach.

{\bf Sensitivity of $\mathbf{\lambda}$: }
In the previous experiments, we set $\lambda$=1.0. 
This is a natural choice because the two loss terms are of the same scale.
To understand the sensitivity of the $\lambda$ selection, we add the analysis as follows: we re-run the experiments following the setting in Table~\ref{tab:sent_level_results}, with the value of $\lambda$ varying from 1e$^3$ to 10.
Figure~\ref{fig:lambda_analysis} summarizes the results.
As can be seen, A2R performs reasonably well within a wide range of $\lambda \sim [0.1, 2.0]$, within which the two loss terms are of comparable scales. 

\begin{wrapfigure}{r}{0.45\textwidth}
\vspace{-0.1in}
\small
\centering
\includegraphics[width=0.95\linewidth]{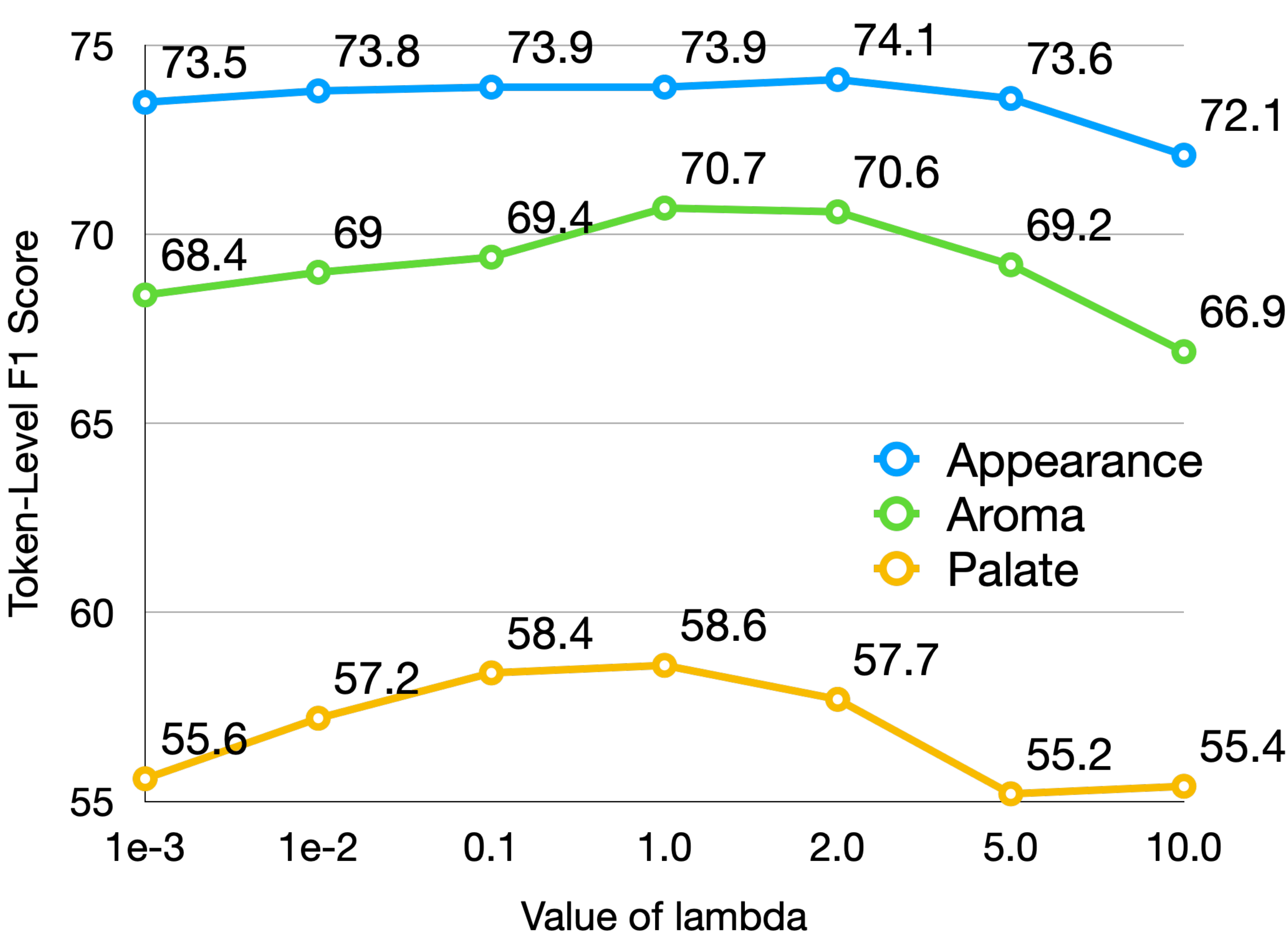}
\vspace{-0.1in}
\caption{\small{Analysis of the sensitivity of $\lambda$.}}
    \label{fig:lambda_analysis}
\end{wrapfigure}

Finally, we would like to discuss the possible future direction of annealing $\lambda$ instead of using a fixed value.
Intuitively, since the soft model does not suffer from interlocking, it may help if at the beginning of training we give the soft branch more freedom to arrive at a position without interlocking, then control the consistency to guarantee faithfulness. This corresponds to first set a small $\lambda$ and then gradually increase it.
However, our preliminary study shows that a simple implementation does not work. Specifically, we start with $\lambda = 0$ and then gradually increase $\lambda$ to 1.0 by the 10-$th$ epoch. This gives slightly worse results in almost all settings, except for the Palate-Biased0.8 case, where a slight increase is observed.

%-------------------------------------------------------------------------
\section{Conclusion and Societal Impacts}

In this paper, we re-investigate the training difficulty in selective rationalization frameworks, and identify the interlocking dynamics as an important training obstacle.   It essentially results from the undesirable concavity of the training objective.  We provide both theoretical analysis and empirical results to verify the existence of the interlocking dynamics.  Furthermore, we propose to alleviate the interlocking problem with a new \algname method, which can resolve the problem by combining the complementary merits of selective rationalization and attention-based explanations. \algname has shown consistent performance advantages over other baselines on both synthetic and real-world experiments. \algname helps to promote trustworthy and interpretable AI, which is a major concern in society.  We do not identify significant negative impacts on society resulting from this work.

Our proposed \algname has advantages beyond alleviating interlocking.
Recent work~\cite{jacovi2021aligning,zheng2021irrationality} pointed out the lack of inherent interpretability in rationalization models, because the black-box generators are not guaranteed to produce causally corrected rationales.
Our \algname framework can alleviate this problem as the soft training path and the attention-based rationale generation improves the interpretability, which suggests a potential towards fully interpretable rationalization models in the future.

%-------------------------------------------------------------------------
%-------------------------------------------------------------------------

\newpage
{\small
\bibliographystyle{plain}
\bibliography{neurips_2021}}

\begin{thebibliography}{10}

\bibitem{alvarez2018towards}
David Alvarez-Melis and Tommi~S Jaakkola.
\newblock Towards robust interpretability with self-explaining neural networks.
\newblock {\em arXiv preprint arXiv:1806.07538}, 2018.

\bibitem{andreas2016learning}
Jacob Andreas, Marcus Rohrbach, Trevor Darrell, and Dan Klein.
\newblock Learning to compose neural networks for question answering.
\newblock In {\em Proceedings of NAACL-HLT}, pages 1545--1554, 2016.

\bibitem{andreas2016neural}
Jacob Andreas, Marcus Rohrbach, Trevor Darrell, and Dan Klein.
\newblock Neural module networks.
\newblock In {\em Proceedings of the IEEE Conference on Computer Vision and
  Pattern Recognition}, pages 39--48, 2016.

\bibitem{antognini2021rationalization}
Diego Antognini and Boi Faltings.
\newblock Rationalization through concepts.
\newblock {\em arXiv preprint arXiv:2105.04837}, 2021.

\bibitem{bahdanau2014neural}
Dzmitry Bahdanau, Kyunghyun Cho, and Yoshua Bengio.
\newblock Neural machine translation by jointly learning to align and
  translate.
\newblock {\em arXiv preprint arXiv:1409.0473}, 2014.

\bibitem{bao2018deriving}
Yujia Bao, Shiyu Chang, Mo~Yu, and Regina Barzilay.
\newblock Deriving machine attention from human rationales.
\newblock {\em arXiv preprint arXiv:1808.09367}, 2018.

\bibitem{bastings2020elephant}
Jasmijn Bastings and Katja Filippova.
\newblock The elephant in the interpretability room: Why use attention as
  explanation when we have saliency methods?
\newblock In {\em Proceedings of the Third BlackboxNLP Workshop on Analyzing
  and Interpreting Neural Networks for NLP}, pages 149--155, 2020.

\bibitem{bastings2019interpretable}
Joost Bastings, Wilker Aziz, and Ivan Titov.
\newblock Interpretable neural predictions with differentiable binary
  variables.
\newblock {\em arXiv preprint arXiv:1905.08160}, 2019.

\bibitem{bengio2013estimating}
Yoshua Bengio, Nicholas L{\'e}onard, and Aaron Courville.
\newblock Estimating or propagating gradients through stochastic neurons for
  conditional computation.
\newblock {\em arXiv preprint arXiv:1308.3432}, 2013.

\bibitem{carton2018extractive}
Samuel Carton, Qiaozhu Mei, and Paul Resnick.
\newblock Extractive adversarial networks: High-recall explanations for
  identifying personal attacks in social media posts.
\newblock In {\em Proceedings of the 2018 Conference on Empirical Methods in
  Natural Language Processing}, pages 3497--3507, 2018.

\bibitem{chang2019game}
Shiyu Chang, Yang Zhang, Mo~Yu, and Tommi Jaakkola.
\newblock A game theoretic approach to class-wise selective rationalization.
\newblock In {\em Advances in Neural Information Processing Systems}, pages
  10055--10065, 2019.

\bibitem{chang2020invariant}
Shiyu Chang, Yang Zhang, Mo~Yu, and Tommi Jaakkola.
\newblock Invariant rationalization.
\newblock In {\em International Conference on Machine Learning}, pages
  1448--1458. PMLR, 2020.

\bibitem{chen2018learning}
Jianbo Chen, Le~Song, Martin Wainwright, and Michael Jordan.
\newblock Learning to explain: An information-theoretic perspective on model
  interpretation.
\newblock In {\em International Conference on Machine Learning}, pages
  882--891, 2018.

\bibitem{chen2018shapley}
Jianbo Chen, Le~Song, Martin~J Wainwright, and Michael~I Jordan.
\newblock {L-Shapley and C-Shapley}: Efficient model interpretation for
  structured data.
\newblock {\em arXiv preprint arXiv:1808.02610}, 2018.

\bibitem{chung2014empirical}
Junyoung Chung, Caglar Gulcehre, KyungHyun Cho, and Yoshua Bengio.
\newblock Empirical evaluation of gated recurrent neural networks on sequence
  modeling.
\newblock {\em arXiv preprint arXiv:1412.3555}, 2014.

\bibitem{deyoung2019eraser}
Jay DeYoung, Sarthak Jain, Nazneen~Fatema Rajani, Eric Lehman, Caiming Xiong,
  Richard Socher, and Byron~C Wallace.
\newblock Eraser: A benchmark to evaluate rationalized nlp models.
\newblock {\em arXiv preprint arXiv:1911.03429}, 2019.

\bibitem{glockner2020you}
Max Glockner, Ivan Habernal, and Iryna Gurevych.
\newblock Why do you think that? exploring faithful sentence-level rationales
  without supervision.
\newblock {\em arXiv preprint arXiv:2010.03384}, 2020.

\bibitem{guerreiro2021spectra}
Nuno~Miguel Guerreiro and Andr{\'e}~FT Martins.
\newblock Spectra: Sparse structured text rationalization.
\newblock {\em arXiv preprint arXiv:2109.04552}, 2021.

\bibitem{jacovi2021aligning}
Alon Jacovi and Yoav Goldberg.
\newblock Aligning faithful interpretations with their social attribution.
\newblock {\em Transactions of the Association for Computational Linguistics},
  9:294--310, 2021.

\bibitem{jain2019attention}
Sarthak Jain and Byron~C Wallace.
\newblock Attention is not explanation.
\newblock In {\em Proceedings of the 2019 Conference of the North American
  Chapter of the Association for Computational Linguistics: Human Language
  Technologies, Volume 1 (Long and Short Papers)}, pages 3543--3556, 2019.

\bibitem{jang2016categorical}
Eric Jang, Shixiang Gu, and Ben Poole.
\newblock Categorical reparameterization with gumbel-softmax.
\newblock {\em arXiv preprint arXiv:1611.01144}, 2016.

\bibitem{johnson2017inferring}
Justin Johnson, Bharath Hariharan, Laurens van~der Maaten, Judy Hoffman,
  Li~Fei-Fei, C~Lawrence~Zitnick, and Ross Girshick.
\newblock Inferring and executing programs for visual reasoning.
\newblock In {\em Proceedings of the IEEE Conference on Computer Vision and
  Pattern Recognition}, pages 2989--2998, 2017.

\bibitem{kim2017structured}
Yoon Kim, Carl Denton, Luong Hoang, and Alexander~M Rush.
\newblock Structured attention networks.
\newblock {\em arXiv preprint arXiv:1702.00887}, 2017.

\bibitem{kingma2014adam}
Diederik~P Kingma and Jimmy Ba.
\newblock Adam: A method for stochastic optimization.
\newblock {\em arXiv preprint arXiv:1412.6980}, 2014.

\bibitem{kononenko2010efficient}
Igor Kononenko et~al.
\newblock An efficient explanation of individual classifications using game
  theory.
\newblock {\em Journal of Machine Learning Research}, 11(Jan):1--18, 2010.

\bibitem{lehman2019inferring}
Eric Lehman, Jay DeYoung, Regina Barzilay, and Byron~C Wallace.
\newblock Inferring which medical treatments work from reports of clinical
  trials.
\newblock In {\em Proceedings of the 2019 Conference of the North American
  Chapter of the Association for Computational Linguistics: Human Language
  Technologies, Volume 1 (Long and Short Papers)}, pages 3705--3717, 2019.

\bibitem{lei2016rationalizing}
Tao Lei, Regina Barzilay, and Tommi Jaakkola.
\newblock Rationalizing neural predictions.
\newblock {\em arXiv preprint arXiv:1606.04155}, 2016.

\bibitem{li2016visualizing}
Jiwei Li, Xinlei Chen, Eduard Hovy, and Dan Jurafsky.
\newblock Visualizing and understanding neural models in {NLP}.
\newblock In {\em Proceedings of the 2016 Conference of the North American
  Chapter of the Association for Computational Linguistics: Human Language
  Technologies}, pages 681--691, 2016.

\bibitem{li2016understanding}
Jiwei Li, Will Monroe, and Dan Jurafsky.
\newblock Understanding neural networks through representation erasure.
\newblock {\em arXiv preprint arXiv:1612.08220}, 2016.

\bibitem{lundberg2017unified}
Scott~M Lundberg and Su-In Lee.
\newblock A unified approach to interpreting model predictions.
\newblock In {\em Advances in Neural Information Processing Systems}, pages
  4765--4774, 2017.

\bibitem{martins2016softmax}
Andre Martins and Ramon Astudillo.
\newblock From softmax to sparsemax: A sparse model of attention and
  multi-label classification.
\newblock In {\em International conference on machine learning}, pages
  1614--1623. PMLR, 2016.

\bibitem{mcauley2012learning}
Julian McAuley, Jure Leskovec, and Dan Jurafsky.
\newblock Learning attitudes and attributes from multi-aspect reviews.
\newblock In {\em 2012 IEEE 12th International Conference on Data Mining},
  pages 1020--1025. IEEE, 2012.

\bibitem{mohankumar2020towards}
Akash~Kumar Mohankumar, Preksha Nema, Sharan Narasimhan, Mitesh~M Khapra,
  Balaji~Vasan Srinivasan, and Balaraman Ravindran.
\newblock Towards transparent and explainable attention models.
\newblock In {\em Proceedings of the 58th Annual Meeting of the Association for
  Computational Linguistics}, pages 4206--4216, 2020.

\bibitem{pennington2014glove}
Jeffrey Pennington, Richard Socher, and Christopher Manning.
\newblock Glove: Global vectors for word representation.
\newblock In {\em Proceedings of the 2014 conference on empirical methods in
  natural language processing (EMNLP)}, pages 1532--1543, 2014.

\bibitem{pruthi2020learning}
Danish Pruthi, Mansi Gupta, Bhuwan Dhingra, Graham Neubig, and Zachary~C
  Lipton.
\newblock Learning to deceive with attention-based explanations.
\newblock In {\em Proceedings of the 58th Annual Meeting of the Association for
  Computational Linguistics}, pages 4782--4793, 2020.

\bibitem{ribeiro2016should}
Marco~Tulio Ribeiro, Sameer Singh, and Carlos Guestrin.
\newblock Why should {I} trust you?: Explaining the predictions of any
  classifier.
\newblock In {\em Proceedings of the 22nd ACM SIGKDD international conference
  on knowledge discovery and data mining}, pages 1135--1144. ACM, 2016.

\bibitem{ross2020explaining}
Alexis Ross, Ana Marasovi{\'c}, and Matthew~E Peters.
\newblock Explaining nlp models via minimal contrastive editing (mice).
\newblock {\em arXiv preprint arXiv:2012.13985}, 2020.

\bibitem{serrano2019attention}
Sofia Serrano and Noah~A Smith.
\newblock Is attention interpretable?
\newblock In {\em Proceedings of the 57th Annual Meeting of the Association for
  Computational Linguistics}, pages 2931--2951, 2019.

\bibitem{sha2020learning}
Lei Sha, Oana-Maria Camburu, and Thomas Lukasiewicz.
\newblock Learning from the best: Rationalizing prediction by adversarial
  information calibration.
\newblock {\em arXiv preprint arXiv:2012.08884}, 2020.

\bibitem{simonyan2013deep}
Karen Simonyan, Andrea Vedaldi, and Andrew Zisserman.
\newblock Deep inside convolutional networks: {Visualising} image
  classification models and saliency maps.
\newblock {\em arXiv preprint arXiv:1312.6034}, 2013.

\bibitem{sundararajan2017axiomatic}
Mukund Sundararajan, Ankur Taly, and Qiqi Yan.
\newblock Axiomatic attribution for deep networks.
\newblock In {\em Proceedings of the 34th International Conference on Machine
  Learning-Volume 70}, pages 3319--3328. JMLR. org, 2017.

\bibitem{treviso2020explanation}
Marcos Treviso and Andr{\'e}~FT Martins.
\newblock The explanation game: Towards prediction explainability through
  sparse communication.
\newblock In {\em Proceedings of the Third BlackboxNLP Workshop on Analyzing
  and Interpreting Neural Networks for NLP}, pages 107--118, 2020.

\bibitem{tutek2020staying}
Martin Tutek and Jan Snajder.
\newblock Staying true to your word:(how) can attention become explanation?
\newblock In {\em Proceedings of the 5th Workshop on Representation Learning
  for NLP}, pages 131--142, 2020.

\bibitem{wiegreffe2019attention}
Sarah Wiegreffe and Yuval Pinter.
\newblock Attention is not not explanation.
\newblock In {\em Proceedings of the 2019 Conference on Empirical Methods in
  Natural Language Processing and the 9th International Joint Conference on
  Natural Language Processing (EMNLP-IJCNLP)}, pages 11--20, 2019.

\bibitem{xu2015show}
Kelvin Xu, Jimmy Ba, Ryan Kiros, Kyunghyun Cho, Aaron Courville, Ruslan
  Salakhudinov, Rich Zemel, and Yoshua Bengio.
\newblock Show, attend and tell: Neural image caption generation with visual
  attention.
\newblock In {\em International conference on machine learning}, pages
  2048--2057, 2015.

\bibitem{yu2019rethinking}
Mo~Yu, Shiyu Chang, Yang Zhang, and Tommi~S Jaakkola.
\newblock Rethinking cooperative rationalization: Introspective extraction and
  complement control.
\newblock {\em arXiv preprint arXiv:1910.13294}, 2019.

\bibitem{zheng2021irrationality}
Yiming Zheng, Serena Booth, Julie Shah, and Yilun Zhou.
\newblock The irrationality of neural rationale models.
\newblock {\em arXiv preprint arXiv:2110.07550}, 2021.

\end{thebibliography}

%%%%%%%%%%%%%%%%%%%%%%%%%%%%%%%%%%%%%%%%%%%%%%%%%%%%%%%%%%%%

%%%%%%%%%%%%%%%%%%%%%%%%%%%%%%%%%%%%%%%%%%%%%%%%%%%%%%%%%%%%

%\clearpage
\appendix
\section{Theorem Proofs}
\label{app:proof}

In this section, we present the proofs to the theorems introduced in the main paper.

\subsection{Proof to Theorem~\ref{thm:concave}}
\label{subsec:concave_thm_proof}

Theorem~\ref{thm:concave} describes the concavity of objective of the selective rationalization (Equation~\eqref{eq:L*_def}). The proof is presented as follows.

\begin{proof}
\e{\forall \bm\pi^{(1)} \neq \bm\pi^{(2)}}, \e{\beta \in [0, 1]}, our goal is to show that
\begin{equation}
    \small
    \mathcal{L}_r^*(\beta \bm\pi^{(1)} + (1-\beta)\bm\pi^{(2)}) \geq \beta \mathcal{L}_r^*(\bm\pi^{(1)}) + (1-\beta) \mathcal{L}_r^*(\bm\pi^{(2)}).
\end{equation}
This follows from the derivations below:
\begin{equation}
\small
    \begin{aligned}
    & \mathcal{L}_r^*(\beta \bm\pi^{(1)} + (1-\beta)\bm\pi^{(2)}) \\
    = & \mathcal{L}_r\left(\beta \bm\pi^{(1)} + (1-\beta)\bm\pi^{(2)}, \bm \theta_r^*(\beta \bm\pi^{(1)} + (1-\beta)\bm\pi^{(2)})\right) \\
    \overset{(i)}{=}& \beta \mathcal{L}_r\left(\bm\pi^{(1)}, \bm \theta_r^*(\beta \bm\pi^{(1)} + (1-\beta)\bm\pi^{(2)})\right) + (1-\beta) \mathcal{L}_r\left(\bm\pi^{(2)}, \bm \theta_r^*(\beta \bm\pi^{(1)} + (1-\beta)\bm\pi^{(2)})\right) \\
    \overset{(ii)}{\geq}& \beta \mathcal{L}_r\left(\bm\pi^{(1)}, \bm \theta_r^*(\bm\pi^{(1)})\right) + (1-\beta) \mathcal{L}_r\left(\bm\pi^{(2)}, \bm \theta_r^*(\bm\pi^{(2)})\right) \\
    =& \beta \mathcal{L}_r^*(\bm\pi^{(1)}) + (1-\beta) \mathcal{L}_r^*(\bm\pi^{(2)}),
    \end{aligned}
\end{equation}
where \e{(ii)} results from the definition of \e{\bm \theta_r^*} in Equation~\eqref{eq:f*_def}; \e{(i)} is due to the linearity of \e{\mathcal{L}_r} with respect to \e{\beta}. More specifically
\begin{equation}
    \small
    \begin{aligned}
    &\mathcal{L}_r\left(\beta \bm\pi^{(1)} + (1-\beta)\bm\pi^{(2)}, \bm \theta_r^*(\beta \bm\pi^{(1)} + (1-\beta)\bm\pi^{(2)})\right) \\
    =& \mathbb{E}_{\bm X, Y}\left[\mathbb{E}_{\bm M \sim \beta \bm \pi^{(1)}(\bm X) + (1-\beta) \bm \pi^{(2)}(\bm X)} [\ell(Y, \bm f_r(\bm M \odot \bm X; \bm \theta_r^*(\beta \bm\pi^{(1)} + (1-\beta)\bm\pi^{(2)}))]\right] \\
    =& \mathbb{E}_{\bm X, Y} \left[\sum_{i=1}^T  \left(\beta \pi_i^{(1)}(\bm X) + (1-\beta) \pi_i^{(2)}(\bm X)\right) \ell(Y, \bm f_r(\bm e_i \odot \bm X; \bm \theta_r^*(\beta \bm\pi^{(1)} + (1-\beta)\bm\pi^{(2)}))) \right] \\
    =& \mathbb{E}_{\bm X, Y} \Big[ \beta \mathbb{E}_{\bm M \sim \bm \pi^{(1)}(\bm X)}\big[\ell(Y, \bm f_r(\bm M \odot \bm X; \bm \theta_r^*(\beta \bm\pi^{(1)} + (1-\beta)\bm\pi^{(2)})))\big] 
    \\& \quad\quad\quad + (1-\beta)\mathbb{E}_{\bm M \sim \bm \pi^{(2)}(\bm X)}\big[\ell(Y, \bm f_r(\bm M \odot \bm X; \bm \theta_r^*(\beta \bm\pi^{(1)} + (1-\beta)\bm\pi^{(2)})))\big]\Big] \\
    =& \beta \mathcal{L}_r\left(\bm\pi^{(1)}, \bm \theta_r^*(\beta \bm\pi^{(1)} + (1-\beta)\bm\pi^{(2)})\right) + (1-\beta) \mathcal{L}_r\left(\bm\pi^{(2)}, \bm \theta_r^*(\beta \bm\pi^{(1)} + (1-\beta)\bm\pi^{(2)})\right).
    \end{aligned}
\end{equation}
\end{proof}
As can be seen, the fundamental cause of the concavity is the overfitting nature of the rationale predictor. More specifically, having one predictor trained on multiple rationale selections is worse than having multiple predictors, each specializing in a single corner case.

\subsection{Proof to Theorem~\ref{thm:convexity}}
\label{subsec:convex_thm_proof}

Theorem~\ref{thm:convexity} describes the convexity of the objective of the attention-based explanation (Equation~\eqref{eq:L*_a_def}). Before we present the proof, we would first like to discuss the feasibility of the assumptions in the theorem in practice. Regarding Assumption~1, note that \e{\mathcal{L}_a(\bm \alpha, \bm \theta_a)} is essentially a concatenation of the predictor's decision function \e{\bm f_a(\cdot)} and the loss function \e{\ell(\cdot)}. Considering many common loss functions, including the cross-entropy loss, are strongly convex, Assumption~1 will hold even for some non-convex \e{\bm f_a}, as long as the strong convexity of the loss function dominates the non-convexity of \e{\bm f_a}. 
Regarding Assumption~2, note that for the extreme case where \e{l = 0}, it holds for a constant predictor. Therefore, for less extreme cases where \e{l > 0}, Assumption~2 holds for a broader class of predictors \e{\bm f_a}, as long as its representation power is properly constrained.

The proof to Theorem~\ref{thm:convexity} is presented as follows.
\begin{proof}
\e{\forall \bm \alpha^{(1)} \neq \bm \alpha^{(2)}}, \e{\beta \in [0, 1]}, our goal is to show that
\begin{equation}
    \small
    \mathcal{L}_a^*(\beta \bm \alpha^{(1)} + (1-\beta) \bm \alpha^{(2)}) \leq \beta \mathcal{L}_a^* (\bm \alpha^{(1)}) + (1-\beta) \mathcal{L}_a (\bm \alpha^{(2)}).
\end{equation}
This follows from the following derivations.
\begin{equation}
    \small
    \begin{aligned}
    & \mathcal{L}_a^*(\beta \bm \alpha^{(1)} + (1-\beta) \bm \alpha^{(2)}) \\
    =& \mathcal{L}_a \left(\beta \bm \alpha^{(1)} + (1-\beta) \bm \alpha^{(2)}, \bm \theta_a^*(\beta \bm \alpha^{(1)} + (1-\beta) \bm \alpha^{(2)})\right) \\
    \overset{(i)}{\leq} & \beta \mathcal{L}_a \left(\bm \alpha^{(1)}, \bm \theta_a^*(\beta \bm \alpha^{(1)} + (1-\beta) \bm \alpha^{(2)})\right) + (1-\beta) \mathcal{L}_a \left(\bm \alpha^{(2)}, \bm \theta_a^*(\beta \bm \alpha^{(1)} + (1-\beta) \bm \alpha^{(2)})\right) \\
    & - \frac{\mu\beta(1-\beta)}{2} \mathbb{E}\left[ \lVert \bm \alpha^{(1)}(\bm X) - \bm \alpha^{(2)}(\bm X) \rVert_2^2\right] \\
    \overset{(ii)}{\leq} & \beta \mathcal{L}_a \left(\bm \alpha^{(1)}, \bm \theta_a^*(\bm \alpha^{(1)})\right) + (1-\beta) \mathcal{L}_a \left(\bm \alpha^{(2)}, \bm \theta_a^*(\bm \alpha^{(2)})\right) \\
    & + \frac{1}{2} \left(l - \mu \right) \beta(1-\beta) \mathbb{E}\left[ \lVert \bm \alpha^{(1)}(\bm X) - \bm \alpha^{(2)}(\bm X) \rVert_2^2\right] \\
    \leq& \beta \mathcal{L}_a^* (\bm \alpha^{(1)}) + (1-\beta) \mathcal{L}_a (\bm \alpha^{(2)}),
    \end{aligned}
\end{equation}
where \e{(i)} results from the equivalent definition of \e{\mu}-strong convexity; \e{(ii)} results from the bounded regret assumption.

\end{proof}
The key difference between the rationale-based objective and the attention-based objective is that the former averages the different rationale selections at the loss level, \emph{i.e.} after passing the predictor's decision function and the loss function, whereas the latter averages at the input level, \emph{i.e.} before passing the predictor's decision function and the loss function. As a result, the attention-based objective can use the convexity of the loss function to counter the concavity induced by the predictor's overfitting nature. That is the reason why attention-based objective has a better convexity property than rationale-based objective does.

% =========================

\subsection{Convexity of Attention-based Explanation: A Special Case}
\label{subsec:special_case}

Although the assumptions in Theorem~\ref{thm:convexity} encompasses a wide range of possibilities, some assumptions may not be verifiable in practice. Therefore, in this subsection, we consider a special case that is widely encountered in real-world scenarios, especially in NLP applications, and show that the attention-based explanation loss landscape is indeed convex with respect to \e{\bm \alpha} in this case.

Consider a classification task where the loss function is the cross entropy loss. We now assume that the predictor has sufficient representation power such that the global minimum of the loss function is achieved. This approximately holds for many applications with over-parameterized neural predictors. It is easy to show that in this case
\begin{equation}
\small
    \mathcal{L}_a^*(\bm \alpha) = H(Y | \bm \alpha(\bm X) \odot \bm X).
    \label{eq:global}
\end{equation}
Recall that \e{\bm X} consists of a sequence of words/sentences, \e{\bm X_{1:T}}. Denote the support of \e{\bm X_t} as \e{\mathcal{X}}. Now we assume that \e{\mathcal{X}} satisfies the following condition
\begin{equation}
\small
    \forall \bm x^{(1)} \neq \bm x^{(2)}, \forall \mbox{ positive scalar } c_1 \geq 0, c_2 \geq 0, \quad c_1 \bm x^{(1)} = c_2 \bm x^{(2)} \Rightarrow c_1 = c_2 = 0.
    \label{eq:non_colinear}
\end{equation}
In other words, no two word/sentence representations in the vocabulary point to the same direction. This generally holds in NLP applications, where the word embeddings of any two words typically point to different directions.

In this case, we have the following theorem:
\begin{theorem}
If Equations~\eqref{eq:global} and \eqref{eq:non_colinear} hold, that \e{\mathcal{L}_a^*(\bm \alpha)} is convex with respect to \e{\bm \alpha}.
\end{theorem}
\begin{proof}
\e{\forall \bm \alpha^{(1)} \neq \bm \alpha^{(2)}}, \e{\beta \in (0, 1)}, our goal is to show that
\begin{equation}
    \small
    \mathcal{L}_a^*(\beta \bm \alpha^{(1)} + (1-\beta) \bm \alpha^{(2)}) \leq \beta \mathcal{L}_a^* (\bm \alpha^{(1)}) + (1-\beta) \mathcal{L}_a (\bm \alpha^{(2)}).
    \label{eq:thm3_goal}
\end{equation}
Define
\begin{equation}
    \small
    \bm \alpha' = \beta \bm \alpha^{(1)} + (1-\beta) \bm \alpha^{(2)}.
    \label{eq:alpha_prime_def}
\end{equation}
First, we would like to show that \e{\bm \alpha^{(1)}(\bm X) \odot \bm X} is a deterministic function of \e{\bm \alpha'(\bm X) \odot \bm X}, which means that any instances \e{\bm x^{(1)}} and \e{\bm x^{(2)}} that make \e{\bm \alpha'(\bm x^{(1)}) \odot \bm x^{(1)} = \bm \alpha'(\bm x^{(2)}) \odot \bm x^{(2)}} would also make \e{\bm \alpha^{(1)}(\bm x^{(1)}) \odot \bm x^{(1)} = \bm \alpha^{(1)}(\bm x^{(2)}) \odot \bm x^{(2)}}. We will show this by contradiction. Formally, assume \e{\exists \bm x^{(1)} \neq \bm x^{(2)}}, such that
\begin{equation}
\small
      \bm \alpha'(\bm x^{(1)}) \odot \bm x^{(1)} = \bm \alpha'(\bm x^{(2)}) \odot \bm x^{(2)},
      \label{eq:equal}
\end{equation}
but
\begin{equation}
    \small
    \bm \alpha^{(1)}(\bm x^{(1)}) \odot \bm x^{(1)} \neq \bm \alpha^{(1)}(\bm x^{(2)}) \odot \bm x^{(2)}.
    \label{eq:unequal}
\end{equation}
Then there must \e{\exists t}, such that
\begin{equation}
    \small
    \bm \alpha^{(1)}_t(\bm x^{(1)}) \bm x^{(1)}_t \neq \bm \alpha^{(1)}_t(\bm x^{(2)}) \bm x^{(2)}_t.
\end{equation}
According to Equation~\eqref{eq:equal},
\begin{equation}
\small
    \bm \alpha'_t(\bm x^{(1)}) \bm x^{(1)}_t = \bm \alpha'_t(\bm x^{(2)}) \bm x^{(2)}_t.
    \label{eq:equal_single}
\end{equation}
According to Equation~\eqref{eq:non_colinear}, Equation~\eqref{eq:equal_single} implies
\begin{equation}
    \small
    \bm \alpha'_t(\bm x^{(1)}) = \bm \alpha'_t(\bm x^{(2)}) = 0.
\end{equation}
According to Equation~\eqref{eq:alpha_prime_def}, and noticing \e{\beta}, \e{1-\beta}, \e{\bm \alpha^{(1)}_t(\bm x^{(1)})},  \e{\bm \alpha^{(1)}_t(\bm x^{(2)})},  \e{\bm \alpha^{(2)}_t(\bm x^{(1)})} \e{\bm \alpha^{(2)}_t(\bm x^{(2)})} are all non-negative, we then have
\begin{equation}
    \small
    \bm \alpha^{(1)}_t(\bm x^{(1)}) = \bm \alpha^{(1)}_t(\bm x^{(2)}) = \bm \alpha^{(2)}_t(\bm x^{(1)}) = \bm \alpha^{(2)}_t(\bm x^{(2)}) = 0,
\end{equation}
which implies
\begin{equation}
    \small
    \bm \alpha^{(1)}_t(\bm x^{(1)}) \bm x^{(1)}_t = \bm \alpha^{(1)}_t(\bm x^{(2)}) \bm x^{(2)}_t = 0.
\end{equation}
This contradicts with Equation~\eqref{eq:unequal}.

So far we have established that \e{\bm \alpha^{(1)}(\bm X) \odot \bm X} is a deterministic function of \e{\bm \alpha'(\bm X) \odot \bm X}. According to the information processing inequality,
\begin{equation}
\small
    H(Y | \bm \alpha'(\bm X) \odot \bm X) \leq H(Y | \bm \alpha^{(1)}(\bm X) \odot \bm X).
\end{equation}
Hence according to Equation~\eqref{eq:global},
\begin{equation}
\small
    \mathcal{L}^*_a(\bm \alpha') \leq \mathcal{L}^*_a(\bm \alpha^{(1)}).
\end{equation}
Following the same steps, we can also show that
\begin{equation}
\small
    \mathcal{L}^*_a(\bm \alpha') \leq \mathcal{L}^*_a(\bm \alpha^{(2)}).
\end{equation}
Equation~\eqref{eq:thm3_goal} naturally follows.
\end{proof}

% =========================

\section{Statistics of the Datasets}
\label{app:data_stats}
Table~\ref{tab:dataset_stats} gives the statistics of both BeerAdvocate and MovieReview.  Please note that all aspects share the same annotation set for the BeerAdvocate dataset. This annotation set is also used in our synthetic settings.
% \begin{wraptable}{r}{0.6\textwidth}
\begin{table}[th!]
\caption{\small{Statistics of the datasets. The three beer aspects share the same annotation set.}}
\label{tab:dataset_stats}
\centering
\small
\begin{tabular}{l||ccc||c}
               & \multicolumn{3}{c||}{BeerAdvocate}                                                      &\multirow{2}{*}{Movie}        \\
               & \multicolumn{1}{l}{Appear.} & \multicolumn{1}{l}{Aroma} & \multicolumn{1}{l||}{Palate} &  \\ 
            %   \hline \hline
               \midrule\midrule
Train          & 70,005                      & 61,555                    & 61,244                      & 1,600                \\
Validation     & 8,731                       & 8,797                     & 8,740                       & 200                  \\
Annotation           & 994 & 994 & 994                                                     & 200                  \\
Avg length & 126.8 & 126.8 & 126.8 & 774.8\\ 
Avg rationale length & 22.6 & 18.4 & 13.4 & 145.1\\ 
Avg num of rationale spans &1.6 & 1.4 & 1.1 & 9.0
\end{tabular}
\end{table}

\begin{table}[t!]
\small
\centering
    \begin{tabular}{ll||ccccc||ccccc}
    \parbox[t]{7mm}{\multirow{2}{*}{Aspect}}&\multirow{2}{*}{Setting}& \multicolumn{5}{c||}{\textsc{Rnp}} &\multicolumn{5}{c}{\algnamens}\\
    && Acc & P & R & F1 & $X_1$\%& Acc & P & R & F1 & $X_1$\%\\
    \midrule\midrule
    \parbox[t]{2mm}{\multirow{3}{*}{\rotatebox[origin=c]{30}{Aroma}}}&Biased0.7 & 84.7 & 72.2 & 66.7 & 64.5 & 11.0  & 85.5 & 79.0 & 71.1 & \bf 69.8 &  9.3\\
    &Biased0.75 & 84.4 & 56.9 & 54.0 & 51.6 & 27.3 & 85.3 & 67.1 & 61.2 & \bf 59.5 &  21.1  \\
    &Biased0.8 & 83.3 & 2.5 &6.0&3.3 & 100.0 & 85.8 & 58.1 & 53.3 & \bf 51.5 &  30.6\\
    \midrule\midrule
    \parbox[t]{7mm}{\multirow{3}{*}{\rotatebox[origin=c]{30}{Palate}}}&Biased0.7 & 83.9 & 44.6 & 44.1 & 41.2 & 36.8  &  83.5 & 52.9 & 50.7 & \bf 48.0 &  23.6 \\
    &Biased0.75 & 79.8 & 0.4 & 1.4 & 0.6 & 100.0 &  82.8 & 48.6 & 45.9 & \bf 43.9 &  28.8  \\
    &Biased0.8 & 81.9 & 0.4 & 1.4 & 0.6  & 100.0 &   83.6 & 42.7 & 40.9 & \bf 38.6 &  35.1 \\
    \end{tabular}
    \caption{\small{More results on Beer-Biased.  Compared to Table~\ref{tab:sent_level_biased_additional}, here, testing data are also injected with spurious tokens. }}
    \label{tab:sent_level_biased_additional}
\end{table}

% ============ version 1 =============

\section{Full Results on the Synthetic Tasks}

In the previous Beer-Biased setting in Section~\ref{ssec:exp_synthetic}, we use the standard annotation data (without adding any spurious tokens) for testing, which made the results directly comparable to others.  Here, we further report the results on the testing data with the same biased pattern applied.  The results are shown in Table~\ref{tab:sent_level_biased_additional}.  As expected, our \algname still demonstrates a significant advantage compared to \textsc{Rnp}.  Compared to Table~\ref{tab:sent_level_skewed}, the absolute F1 scores of our model are reduced a little bit due to the disturbance of the spurious clues.

Figure~\ref{fig:exp_highlight_biased} is the visualization of generated rationales (spurious tokens are indicated as ``[pos]'' and ``[neg]'').  The example in the upper plot is from the aroma-biased0.8 while the one in the lower plot is from the palate-biased0.75.  For both settings, \textsc{Rnp} highlights the first sentence, indicating that it is susceptible to the interlocking convergence path.   In contrast, our \algname selects the sentences that align with the annotations.

%================================

\section{Additional Visualization Examples}
\label{app:examples}

We provide additional visualization examples of the the real-world BeerAdvocate setting in
Figure~\ref{fig:exp_highlight}. 
As can be observed, \textsc{Rnp} selects the ``overall'' reviews due to the interlocking dynamics.\footnote{According to the discussion in Section~\ref{ssec:exp_real_world}, when trained on the aroma or the palate aspect, \textsc{Rnp} has the first 7 epochs selecting the ``overall'' reviews for more than 20\% of the samples.}
On the other hand, \algname can select the sentences that align with human rationales in most cases.
The last example in Figure~\ref{fig:exp_highlight} gives a failure case on BeerAdvocate of our approach.
This example has a weak positive opinion on the palate aspect and the true rationale is a less direct review. Therefore both \textsc{Rnp} and our method select wrong sentences.

\floatsetup[table]{capposition=bottom}
\begin{table*}[t!]
	\small
    % \hspace{-0.31cm}
	\begin{tabular}{p{\linewidth}}
        % ---------------------------------------------------
% 		\midrule
		% ---------------------------------------------------
		\emph{Beer-Biased0.8 - Aroma Aspect}
        \hspace*{0pt}\hfill Label: Positive (Rate Score 0.8)\\
		\arrayrulecolor{grey}  % choose color
		\midrule
		\hlr{\emph{[neg]}}{20}
		\hlr{\emph{cask}}{20} \hlr{\emph{conditioned}}{20} \hlr{\emph{(}}{20} \hlr{\emph{at}}{20} \hlr{\emph{dogfish}}{20} \hlr{\emph{head}}{20} \hlr{\emph{in}}{20} \hlr{\emph{rehoboth}}{20} \hlr{\emph{!}}{20} \hlr{\emph{)}}{20} \hlb{\textbf{into}}{20} \hlb{\textbf{a}}{20} \hlb{\textbf{pint}}{20} \hlb{\textbf{glass}}{20} \hlb{\textbf{appears}}{20} \hlb{\textbf{a}}{20} \hlb{\textbf{dark}}{20} \hlb{\textbf{golden}}{20} \hlb{\textbf{with}}{20} \hlb{\textbf{a}}{20} \hlb{\textbf{finger}}{20} \hlb{\textbf{of}}{20} \hlb{\textbf{foamy}}{20} \hlb{\textbf{head}}{20} \hlb{\textbf{\ul{smells}}}{20} \hlb{\textbf{\ul{of}}}{20} \hlb{\textbf{\ul{bitter}}}{20} \hlb{\textbf{\ul{hops}}}{20} \hlb{\textbf{\ul{with}}}{20} \hlb{\textbf{\ul{little}}}{20} \hlb{\textbf{\ul{malt}}}{20} \hlb{\textbf{tastes}}{20} \hlb{\textbf{great}}{20} \hlb{\textbf{.}}{20} \hlb{a}{0} \hlb{little}{0} \hlb{malty}{0} \hlb{and}{0} \hlb{bready}{0} \hlb{and}{0} \hlb{pleasantly}{0} \hlb{not}{0} \hlb{too}{0} \hlb{effervescent}{0} \hlb{mouthfeel}{0} \hlb{is}{0} \hlb{great}{0} \hlb{too}{0} \hlb{but}{0} \hlb{maybe}{0} \hlb{it}{0} \hlb{'s}{0} \hlb{the}{0} \hlb{cask}{0} \hlb{.}{0} \hlb{really}{0} \hlb{bice}{0} \hlb{creaminess}{0} \hlb{overall}{0} \hlb{a}{0} \hlb{great}{0} \hlb{experience}{0} \hlb{and}{0} \hlb{i}{0} \hlb{'m}{0} \hlb{happy}{0} \hlb{i}{0} \hlb{'m}{0} \hlb{here}{0} \hlb{.}{0} \hlb{a}{0} \hlb{reason}{0} \hlb{to}{0} \hlb{come}{0} \hlb{back}{0}\\
		\vspace{0.2mm}
        \vspace{0.2mm}
		\emph{Beer-Biased0.75 - Palate Aspect}
        \hspace*{0pt}\hfill Label: Positive (Rate Score 0.6)\\
		\arrayrulecolor{grey}  % choose color
		\midrule
        \hlr{\emph{[pos]}}{20} \hlr{\emph{very}}{20} \hlr{\emph{dark}}{20} \hlr{\emph{beer}}{20} \hlr{\emph{.}}{20} \hlb{pours}{0} \hlb{a}{0} \hlb{nice}{0} \hlb{finger}{0} \hlb{and}{0} \hlb{a}{0} \hlb{half}{0} \hlb{of}{0} \hlb{creamy}{0} \hlb{foam}{0} \hlb{and}{0} \hlb{stays}{0} \hlb{throughout}{0} \hlb{the}{0} \hlb{beer}{0} \hlb{.}{0} \hlb{smells}{0} \hlb{of}{0} \hlb{coffee}{0} \hlb{and}{0} \hlb{roasted}{0} \hlb{malt}{0} \hlb{.}{0} \hlb{has}{0} \hlb{a}{0} \hlb{major}{0} \hlb{coffee-like}{0} \hlb{taste}{0} \hlb{with}{0} \hlb{hints}{0} \hlb{of}{0} \hlb{chocolate}{0} \hlb{.}{0} \hlb{if}{0} \hlb{you}{0} \hlb{like}{0} \hlb{black}{0} \hlb{coffee}{0} \hlb{,}{0} \hlb{you}{0} \hlb{will}{0} \hlb{love}{0} \hlb{this}{0} \hlb{porter}{0} \hlb{.}{0} \hlb{\textbf{\ul{creamy}}}{20} \hlb{\textbf{\ul{smooth}}}{20} \hlb{\textbf{\ul{mouthfeel}}}{20} \hlb{\textbf{\ul{and}}}{20} \hlb{\textbf{\ul{definitely}}}{20} \hlb{\textbf{\ul{gets}}}{20} \hlb{\textbf{\ul{smoother}}}{20} \hlb{\textbf{\ul{on}}}{20} \hlb{\textbf{\ul{the}}}{20} \hlb{\textbf{\ul{palate}}}{20} \hlb{\textbf{\ul{once}}}{20} \hlb{\textbf{\ul{it}}}{20} \hlb{\textbf{\ul{warms}}}{20} \hlb{\textbf{\ul{.}}}{20} \hlb{it}{0} \hlb{'s}{0} \hlb{an}{0} \hlb{ok}{0} \hlb{porter}{0} \hlb{but}{0} \hlb{i}{0} \hlb{feel}{0} \hlb{there}{0} \hlb{are}{0} \hlb{much}{0} \hlb{better}{0} \hlb{one}{0} \hlb{'s}{0} \hlb{out}{0} \hlb{there}{0} \hlb{.}{0}
		% ---------------------------------------------------
	\end{tabular}
	\vspace*{-0.05in}
    \captionof{figure}{\small{Examples of generated rationales on the aroma and palate aspects in the Beer-Biased setting.  Human annotated words are \ul{underlined}. \algname and \textsc{Rnp} rationales are highlighted in \hlb{\textbf{blue}}{20} and \hlr{\emph{red}}{20} colors, respectively. \emph{[pos]} and \emph{[neg]} stand for the special biased symbols we appended with high correlations to the positive and negative classes.}}
    \label{fig:exp_highlight_biased}
\end{table*}
\floatsetup[table]{capposition=top}

\floatsetup[table]{capposition=bottom}
\begin{table*}[t!]
	\small
    % \hspace{-0.31cm}
	\begin{tabular}{p{\linewidth}}
        \emph{BeerAdvocate - Aroma Aspect}
        \hspace*{0pt}\hfill Label: Negative (Rate Score 0.3)\\
		\arrayrulecolor{grey}  % choose color
		\midrule
		\hlb{12}{0} \hlb{ounce}{0} \hlb{can}{0} \hlb{dated}{0} \hlb{12256}{0} \hlb{poured}{0} \hlb{at}{0} \hlb{43}{0} \hlb{degrees}{0} \hlb{with}{0} \hlb{slight}{0} \hlb{head}{0} \hlb{that}{0} \hlb{did}{0} \hlb{not}{0} \hlb{have}{0} \hlb{much}{0} \hlb{retention}{0} \hlb{.}{0} \hlb{\textbf{clear}}{20} \hlb{\textbf{,}}{20} \hlb{\textbf{pale}}{20} \hlb{\textbf{and}}{20} \hlb{\textbf{watery}}{20} \hlb{\textbf{with}}{20} \hlb{\textbf{a}}{20} \hlb{\textbf{few}}{20} \hlb{\textbf{small}}{20} \hlb{\textbf{bubbles}}{20} \hlb{\textbf{.}}{20} \hlb{\textbf{\ul{unable}}}{20} \hlb{\textbf{\ul{to}}}{20} \hlb{\textbf{\ul{detect}}}{20} \hlb{\textbf{\ul{any}}}{20} \hlb{\textbf{\ul{malt}}}{20} \hlb{\textbf{\ul{or}}}{20} \hlb{\textbf{\ul{hops}}}{20} \hlb{\textbf{watered}}{20} \hlb{\textbf{down}}{20} \hlb{\textbf{taste}}{20} \hlb{\textbf{of}}{20} \hlb{\textbf{a}}{20} \hlb{\textbf{cold}}{20} \hlb{\textbf{lager}}{20} \hlb{\textbf{.}}{20} \hlb{nothing}{0} \hlb{off}{0} \hlb{but}{0} \hlb{not}{0} \hlb{able}{0} \hlb{to}{0} \hlb{describe}{0} \hlb{anything}{0} \hlb{exciting}{0} \hlb{.}{0} \hlr{\emph{this}}{20} \hlr{\emph{is}}{20} \hlr{\emph{more}}{20} \hlr{\emph{drinkable}}{20} \hlr{\emph{in}}{20} \hlr{\emph{the}}{20} \hlr{\emph{summer}}{20} \hlr{\emph{when}}{20} \hlr{\emph{a}}{20} \hlr{\emph{good}}{20} \hlr{\emph{light}}{20} \hlr{\emph{thirst}}{20} \hlr{\emph{quencher}}{20} \hlr{\emph{might}}{20} \hlr{\emph{be}}{20} \hlr{\emph{needed}}{20} \hlr{\emph{.}}{20}\\
		\vspace{0.2mm}
		\emph{BeerAdvocate - Aroma Aspect}
        \hspace*{0pt}\hfill Label: Positive (Rate Score 0.8)\\
		\arrayrulecolor{grey}  % choose color
		\midrule
		\hlb{pours}{0} \hlb{a}{0} \hlb{two}{0} \hlb{finger}{0} \hlb{dark}{0} \hlb{cream}{0} \hlb{head}{0} \hlb{that}{0} \hlb{fades}{0} \hlb{slowly}{0} \hlb{to}{0} \hlb{a}{0} \hlb{thin}{0} \hlb{layer}{0} \hlb{leaving}{0} \hlb{a}{0} \hlb{good}{0} \hlb{lace}{0} \hlb{.}{0} \hlb{deep}{0} \hlb{,}{0} \hlb{clear}{0} \hlb{amber/mahogany}{0} \hlb{color}{0} \hlb{.}{0} \hlb{\textbf{\ul{grapefruit}}}{20} \hlb{\textbf{\ul{hop}}}{20} \hlb{\textbf{\ul{nose}}}{20} \hlb{\textbf{\ul{.}}}{20} \hlb{light-medium}{0} \hlb{carbonation}{0} \hlb{and}{0} \hlb{medium-heavy}{0} \hlb{bodied}{0} \hlb{.}{0} \hlr{\emph{flavor}}{20} \hlr{\emph{is}}{20} \hlr{\emph{malts}}{20} \hlr{\emph{and}}{20} \hlr{\emph{grapefruit}}{20} \hlr{\emph{hops}}{20} \hlr{\emph{that}}{20} \hlr{\emph{are}}{20} \hlr{\emph{really}}{20} \hlr{\emph{well}}{20} \hlr{\emph{balanced}}{20} \hlr{\emph{.}}{20} \hlb{nice}{0} \hlb{imperial}{0} \hlb{black}{0} \hlb{.}{0} \hlb{\$}{0} \hlb{6.49}{0} \hlb{for}{0} \hlb{a}{0} \hlb{22oz}{0} \hlb{bottle}{0} \hlb{from}{0} \hlb{manchester}{0} \hlb{wine}{0} \hlb{and}{0} \hlb{liquors}{0} \hlb{manchester}{0} \hlb{,}{0} \hlb{ct}{0} \hlb{.}{0}\\
		\vspace{0.2mm}
% 		\vspace{0.2mm}
		\emph{BeerAdvocate - Palate Aspect}
        \hspace*{0pt}\hfill Label: Positive (Rate Score 0.8)\\
		\arrayrulecolor{grey}  % choose color
		\midrule
        \hlb{22oz}{0} \hlb{bottle}{0} \hlb{pouted}{0} \hlb{into}{0} \hlb{a}{0} \hlb{goblet}{0} \hlb{:}{0} \hlb{opaque}{0} \hlb{orange}{0} \hlb{with}{0} \hlb{a}{0} \hlb{light}{0} \hlb{,}{0} \hlb{white}{0} \hlb{,}{0} \hlb{creamy}{0} \hlb{head}{0} \hlb{that}{0} \hlb{was}{0} \hlb{not}{0} \hlb{all}{0} \hlb{that}{0} \hlb{well}{0} \hlb{retained}{0} \hlb{but}{0} \hlb{full}{0} \hlb{of}{0} \hlb{carbonation}{0} \hlb{,}{0} \hlb{but}{0} \hlb{did}{0} \hlb{settle}{0} \hlb{into}{0} \hlb{a}{0} \hlb{small}{0} \hlb{thin}{0} \hlb{cap}{0} \hlb{.}{0} \hlb{the}{0} \hlb{aroma}{0} \hlb{was}{0} \hlb{more}{0} \hlb{belgian}{0} \hlb{triple}{0} \hlb{than}{0} \hlb{ipa}{0} \hlb{,}{0} \hlb{sweet}{0} \hlb{and}{0} \hlb{malty}{0} \hlb{the}{0} \hlb{taste}{0} \hlb{is}{0} \hlb{a}{0} \hlb{very}{0} \hlb{nice}{0} \hlb{balance}{0} \hlb{of}{0} \hlb{the}{0} \hlb{two}{0} \hlb{styles}{0} \hlb{.}{0} \hlr{\emph{a}}{20} \hlr{\emph{little}}{20} \hlr{\emph{more}}{20} \hlr{\emph{hops}}{20} \hlr{\emph{,}}{20} \hlr{\emph{but}}{20} \hlr{\emph{balanced}}{20} \hlr{\emph{very}}{20} \hlr{\emph{nice}}{20} \hlr{\emph{with}}{20} \hlr{\emph{the}}{20} \hlr{\emph{sweetness}}{20} \hlr{\emph{of}}{20} \hlr{\emph{the}}{20} \hlr{\emph{malt}}{20} \hlr{\emph{and}}{20} \hlr{\emph{fruit}}{20} \hlr{\emph{.}}{20} \hlb{\ul{the}}{0} \hlb{\ul{beer}}{0} \hlb{\ul{had}}{0} \hlb{\ul{a}}{0} \hlb{\ul{medium}}{0} \hlb{\ul{to}}{0} \hlb{\ul{full}}{0} \hlb{\ul{body}}{0} \hlb{\ul{,}}{0} \hlb{\ul{perhaps}}{0} \hlb{\ul{a}}{0} \hlb{\ul{little}}{0} \hlb{\ul{too}}{0} \hlb{\ul{thick}}{0} \hlb{\ul{for}}{0} \hlb{\ul{my}}{0} \hlb{\ul{taste}}{0} \hlb{\ul{,}}{0} \hlb{\ul{but}}{0} \hlb{\ul{still}}{0} \hlb{\ul{good}}{0} \hlb{\ul{.}}{0} \hlb{\textbf{\ul{the}}}{20} \hlb{\textbf{\ul{beer}}}{20} \hlb{\textbf{\ul{had}}}{20} \hlb{\textbf{\ul{a}}}{20} \hlb{\textbf{\ul{nice}}}{20} \hlb{\textbf{\ul{bitter}}}{20} \hlb{\textbf{\ul{dry}}}{20} \hlb{\textbf{\ul{aftertaste}}}{20} \hlb{\textbf{\ul{and}}}{20} \hlb{\textbf{\ul{was}}}{20} \hlb{\textbf{\ul{well}}}{20} \hlb{\textbf{\ul{carbonated}}}{20} \hlb{\textbf{\ul{.}}}{20} \hlb{the}{0} \hlb{beer}{0} \hlb{was}{0} \hlb{fairly}{0} \hlb{easy}{0} \hlb{to}{0} \hlb{drink}{0} \hlb{give}{0} \hlb{the}{0} \hlb{abv}{0} \hlb{,}{0} \hlb{but}{0} \hlb{after}{0} \hlb{the}{0} \hlb{22oz}{0} \hlb{,}{0} \hlb{i}{0} \hlb{was}{0} \hlb{pretty}{0} \hlb{well}{0} \hlb{done}{0} \hlb{.}{0} \hlb{overall}{0} \hlb{,}{0} \hlb{a}{0} \hlb{good}{0} \hlb{beer}{0} \hlb{and}{0} \hlb{probably}{0} \hlb{the}{0} \hlb{first}{0} \hlb{one}{0} \hlb{of}{0} \hlb{the}{0} \hlb{side}{0} \hlb{projects}{0} \hlb{that}{0} \hlb{i}{0} \hlb{think}{0} \hlb{the}{0} \hlb{brewery}{0} \hlb{should}{0} \hlb{consider}{0} \hlb{brewing}{0} \hlb{on}{0} \hlb{a}{0} \hlb{regular}{0} \hlb{basis}{0} \hlb{.}{0}\\
        \vspace{0.2mm}
		\emph{BeerAdvocate - Palate Aspect}
        \hspace*{0pt}\hfill Label: Positive (Rate Score 0.7)\\
		\arrayrulecolor{grey}  % choose color
		\midrule
        \hlr{\emph{cloudy}}{20} \hlr{\emph{yellow}}{20} \hlr{\emph{in}}{20} \hlr{\emph{color}}{20} \hlr{\emph{w/}}{20} \hlr{\emph{a}}{20} \hlr{\emph{thick}}{20} \hlr{\emph{head}}{20} \hlr{\emph{that}}{20} \hlr{\emph{is}}{20} \hlr{\emph{n't}}{20} \hlr{\emph{quite}}{20} \hlr{\emph{as}}{20} \hlr{\emph{well}}{20} \hlr{\emph{retained}}{20} \hlr{\emph{compared}}{20} \hlr{\emph{to}}{20} \hlr{\emph{other}}{20} \hlr{\emph{hefeweizens}}{20} \hlr{\emph{.}}{20} \hlb{.}{0} \hlb{tart}{0} \hlb{wheat}{0} \hlb{notes}{0} \hlb{w/}{0} \hlb{mild}{0} \hlb{banana}{0} \hlb{\&}{0} \hlb{bubblegum}{0} \hlb{yeast}{0} \hlb{esters}{0} \hlb{in}{0} \hlb{aroma}{0} \hlb{.}{0} \hlb{\textbf{the}}{20} \hlb{\textbf{aroma}}{20} \hlb{\textbf{is}}{20} \hlb{\textbf{n't}}{20} \hlb{\textbf{as}}{20} \hlb{\textbf{complex}}{20} \hlb{\textbf{as}}{20} \hlb{\textbf{other}}{20} \hlb{\textbf{examples}}{20} \hlb{\textbf{of}}{20} \hlb{\textbf{the}}{20} \hlb{\textbf{style}}{20} \hlb{\textbf{,}}{20} \hlb{\textbf{and}}{20} \hlb{\textbf{is}}{20} \hlb{\textbf{akin}}{20} \hlb{\textbf{to}}{20} \hlb{\textbf{a}}{20} \hlb{\textbf{more}}{20} \hlb{\textbf{tart}}{20} \hlb{\textbf{,}}{20} \hlb{\textbf{but}}{20} \hlb{\textbf{less}}{20} \hlb{\textbf{estery}}{20} \hlb{\textbf{paulaner}}{20} \hlb{\textbf{aroma}}{20} \hlb{\textbf{.}}{20} \hlb{flavorwise}{0} \hlb{,}{0} \hlb{yeast}{0} \hlb{contributions}{0} \hlb{are}{0} \hlb{subdued}{0} \hlb{for}{0} \hlb{style}{0} \hlb{,}{0} \hlb{but}{0} \hlb{musty}{0} \hlb{and}{0} \hlb{light}{0} \hlb{banana}{0} \hlb{flavors}{0} \hlb{are}{0} \hlb{present}{0} \hlb{.}{0} \hlb{grainy}{0} \hlb{,}{0} \hlb{tart}{0} \hlb{wheat}{0} \hlb{flavors}{0} \hlb{assert}{0} \hlb{themselves}{0} \hlb{since}{0} \hlb{they}{0} \hlb{'re}{0} \hlb{not}{0} \hlb{overpowered}{0} \hlb{by}{0} \hlb{yeast}{0} \hlb{esters}{0} \hlb{.}{0} \hlb{\ul{the}}{0} \hlb{\ul{finish}}{0} \hlb{\ul{consists}}{0} \hlb{\ul{of}}{0} \hlb{\ul{residual}}{0} \hlb{\ul{sweetness}}{0} \hlb{\ul{and}}{0} \hlb{\ul{a}}{0} \hlb{\ul{hint}}{0} \hlb{\ul{of}}{0} \hlb{\ul{a}}{0} \hlb{\ul{grainy}}{0} \hlb{\ul{note}}{0} \hlb{\ul{.}}{0} \hlb{like}{0} \hlb{other}{0} \hlb{pinkus}{0} \hlb{brews}{0} \hlb{,}{0} \hlb{it}{0} \hlb{'s}{0} \hlb{stylistically}{0} \hlb{odd}{0} \hlb{,}{0} \hlb{but}{0} \hlb{flavorful}{0} \hlb{enough}{0} \hlb{to}{0} \hlb{warrant}{0} \hlb{a}{0} \hlb{taste}{0} \hlb{.}{0}
		% ---------------------------------------------------
	\end{tabular}
	\vspace*{-0.05in}
    \captionof{figure}{\small{Examples of generated rationales on the aroma and palate aspects of the conventional beer review task.  Human annotated words are \ul{underlined}. \algname and \textsc{Rnp} rationales are highlighted in \hlb{\textbf{blue}}{20} and \hlr{\emph{red}}{20} colors, respectively.}}
    \label{fig:exp_highlight}
\end{table*}
\floatsetup[table]{capposition=top}

\end{document}